\documentclass[12pt]{article}
\usepackage{amsmath}
\usepackage{graphicx,psfrag,epsf}
\usepackage{enumerate}
\usepackage{natbib}
\usepackage{float}
\usepackage{graphicx} % more modern
\usepackage{subfigure} 

\usepackage{natbib}

\usepackage{hyperref}

\usepackage{natbib}

\usepackage{graphicx} 
\usepackage{amsfonts} 
\usepackage{amsmath}
\usepackage{amsthm}
\usepackage{amssymb}
\usepackage{latexsym}
\usepackage{array}
\usepackage{float}
\usepackage{calc}
\usepackage{perpage}
\usepackage{textcomp}
\usepackage{verbatim} 
\usepackage{marvosym}
\usepackage{hieroglf}
\usepackage{multirow} % can combine rows and columns of tables
\usepackage{rotating} % Can rotate text, can be combined with multirow
\usepackage{dsfont} %for nice indicator

\usepackage[english]{babel}
\usepackage{hyperref} 
\usepackage{color}
\usepackage{tikz}
\newtheorem{theorem}{Theorem}[section]

\newtheorem{lemma}{Lemma}[section]

\newtheorem{definition}{Definition}[section]

\usepackage[textwidth=2cm, textsize=scriptsize]{todonotes}

\newcommand{\real}{\mathbb R}
\def\st{\text{ s.t. }}

\def\E{\mathcal E}

\DeclareMathOperator*{\sign}{sign}

\def\E{\mathbb{E}}

\def\R{\mathbb{R}}

\newcommand{\betatrue}{\beta^*}

\newcommand{\qtrexname}{q-TREX}

\newcommand{\consttrex}{\phi}

\newcommand{\blind}{0}

\usepackage{algorithm}
\usepackage{algorithmic}

\def\swap{\textup{swap}}
\def\Tr{\textup{trace}}
\def\diag{\textup{diag}}
\newcommand{\Ncal}{\mathcal{N}}

\begin{document}

\def\spacingset#1{\renewcommand{\baselinestretch}%
{#1}\small\normalsize} \spacingset{1}

%%%%%%%%%%%%%%%%%%%%%%%%%%%%%%%%%%%%%%%%%%%%%%%%%%%%%%%%%%%%%%%%%%%%%%%%%%%%%%

\if0\blind
{
  \title{\bf Non-convex Global Minimization and False Discovery Rate Control for the TREX}
  \author{Jacob Bien\\
    %\thanks{The authors gratefully acknowledge}\hspace{.2cm}\\
    Department of Biological Statistics and Computational Biology\\
    Department of Statistical Science, Cornell University\\
    Irina Gaynanova \\
    Department of Statistics, Texas A\&M University\\
    Johannes Lederer \\
    Departments of Statistics and Biostatistics,\\ University of Washington, Seattle\\
    Christian L. M\"uller \\
    Simons Center for Data Analysis, Simons Foundation
    }
  \maketitle
} \fi

\if1\blind
{
  \bigskip
  \bigskip
  \bigskip
  \begin{center}
    {\LARGE\bf Title}
\end{center}
  \medskip
} \fi

\bigskip
\begin{abstract}
The TREX is a recently introduced method for performing sparse high-dimensional regression.  
Despite its statistical promise as an alternative to the lasso, square-root lasso, and scaled lasso, the TREX is computationally challenging in that it requires solving a non-convex optimization problem.  
% The original TREX paper suggests a heuristic algorithm, but it is unknown how well this approach does in minimizing the TREX optimization problem.  
This paper shows a remarkable result: despite the non-convexity of the TREX problem, there exists a polynomial-time algorithm that is guaranteed to find the {\bf global minimum}.  This result adds the TREX to a very short list of non-convex optimization problems that can be globally optimized (principal components analysis being a famous example).  After deriving and developing this new approach, we demonstrate that (i) the ability of the preexisting TREX heuristic to reach the global minimum is strongly dependent on the difficulty of the underlying statistical problem, (ii) the new polynomial-time algorithm for TREX permits a novel variable ranking and selection scheme, (iii) this scheme can be incorporated into a rule that controls the false discovery rate (FDR) of included features in the model.  To achieve this last aim, we provide an extension of the results of \citet{Barber2015} to establish that the {\em knockoff filter} framework can be applied to the TREX.  This investigation thus provides both a rare case study of a heuristic for non-convex optimization and a novel way of exploiting non-convexity for statistical inference.  
\end{abstract}

\noindent%
\vfill

\newpage
\spacingset{1.45}% DON'T change the spacing!
\section{Introduction}
\label{sec:introduction}

The lasso \citep{tibshirani1996regression} has become a canonical approach to variable selection and predictive modeling in high-dimensional regression settings.  Given a matrix of features $X\in\real^{n\times p}$ and a response vector $Y\in \real^n$, the lasso is based on solving the regularized least-squares problem,
$$
\min_{\beta\in\real^p}\left\{\|Y-X\beta\|^2_2+\lambda\|\beta\|_1\right\},
$$
where $\lambda\ge0$ is a tuning parameter that controls the sparsity of the solution.  When $Y=X\beta^*+\sigma\varepsilon$ with each $\varepsilon_i$ having zero mean and variance $1$, it has been shown that the lasso has strong performance guarantees in terms of support recovery, estimation, and predictive performance if one takes $\lambda\sim \sigma\|X^\top\varepsilon\|_\infty$. To address the problem of  $\sigma$ being typically unknown, \citet{sqrtlasso,scaledlasso} proposed modifications of the lasso objective function. One can view these modifications as scaling the lasso objective function by an estimate of $\sigma,$ see \citet{LedererMueller:14}:
$$
\min_{\beta\in\real^p}\left\{\frac{\|Y-X\beta\|^2_2}{\frac{1}{\sqrt{n}}\|Y-X\beta\|_2}+\gamma\|\beta\|_1\right\}.
$$
In this way, the optimal tuning parameter $\gamma$ does not depend on $\sigma$. However, since $\varepsilon$ and its distribution are also unknown in practice, \citet{LedererMueller:14} proposed the TREX, which takes the above argument one step further.  Recalling that a theoretically desirable tuning parameter for the lasso is $\lambda\sim \sigma\|X^\top\varepsilon\|_\infty$, they propose to scale the lasso objective by an estimate of this quantity:
\begin{align}
\min_{\beta\in\real^p}\left\{\frac{\|Y-X\beta\|_2^2}{\|X^\top(Y-X\beta)\|_\infty}+\consttrex\|\beta\|_1\right\}.
\label{eq:trex}
\end{align}
The parameter $\consttrex\ge0$, they argue, can be thought of as constant ($\consttrex=1/2$ being the standard choice).  They present several promising examples in which TREX, with no tuning of $\consttrex$, can be effectively used as an alternative to the lasso.

There is, however, a major technical difficulty introduced in the TREX formulation.  Unlike the lasso, square-root lasso, and scaled lasso, the TREX is based on a non-convex optimization problem. The left panel of Figure~\ref{fig:convex-contours} shows the contours of the objective function in \eqref{eq:trex} for a simple example in which $p=2$, revealing a complicated, non-differentiable objective surface with multiple local minima.  

Estimators based on non-convex problems can generally not be computed. Hence, one must typically be satisfied with either (a)~a theoretical estimator that is of limited practical value or (b)~a redefinition of the estimator as the output of a particular algorithm chosen to
approximately (or so one hopes) optimize the objective function. It is rare, but fortunate, when a particular non-convex problem of interest can be efficiently solved
(i.e., globally optimized).  Principal component analysis is one of the few 
examples of a non-convex problem where global optimization is
computationally tractable.

The first term of the TREX optimization problem~\eqref{eq:trex} is
non-convex, and therefore, one might expect that one needs to resort to
either (a) or (b) above.  Indeed, \citet{LedererMueller:14} go the latter route by introducing a 
heuristic scheme in which the $\ell_\infty$-norm is replaced by an
$\ell_q$-norm for some large value of $q$ to yield a differentiable,
though still non-convex, loss function (we will refer to this heuristic as \qtrexname~throughout).  In strong contrast, we derive a  remarkable and surprising result; namely, that the TREX problem, although non-convex, is amenable to polynomial-time {\em global optimization}.  The key to our
approach is the observation that problem \eqref{eq:trex} can be
equivalently expressed as the minimum over $2p$ convex problems.  We
present this reduction in Section~\ref{sec:main-proposal}, and the remainder of the paper exploits this reduction in several directions.

%Another direction in which we exploit the reduction is in developing a deeper insight into the TREX heuristic. 
It is rarely possible to provide a rigorous empirical evaluation for heuristics of non-convex problems since generally the global minimum is impossible to attain certifiably.  We therefore view the TREX as an interesting case study for non-convex heuristics in general and, in Section~\ref{sec:empirical-qtrex}, we capitalize upon our ability to perform global minimization to provide a detailed look at the performance of the \qtrexname~heuristic.

In particular, we find that the q-TREX heuristic's statistical performance (in terms of estimation error) is in fact similar to that of the global minimizer of the TREX objective, even though the estimates themselves can be quite different.  This observation has important practical implications.  In particular, it provides backing for the use of the q-TREX heuristic.  We observe that the q-TREX heuristic is faster to compute than our algorithm for global optimization, so observing similar statistical performance is encouraging since it suggests that we can use q-TREX on large-scale problems without loss in statistical performance.

In Section~\ref{sec:knockoff}, we show how a slight modification of the results  of \citet{Barber2015} about the  {\em knockoff filter} leads to two procedures for (provably) controlling the FDR of features selected based on the TREX.  Interestingly, the empirically more successful of these procedures exploits information about the $2p$ subproblems to design a novel variable ranking scheme.  Thus, while the main contribution of this paper is centered around optimization and computation, this work has interesting statistical implications as well.

%and
%then discuss the convex subproblems in greater depth in Section
%\ref{sec:how-solve-each}.

In Section~\ref{sec:empirical-knockoff}, we provide empirical corroboration of our theoretical result that our TREX-based knockoff filter does in fact provide FDR control.
%Furthermore, we show how we can exploit information about the $2p$ subproblems to design a novel variable ranking scheme for statistical inference. 
%In Section~\ref{sec:hiv}, 
We also apply these new knockoff filters on a large HIV-1 genotype/drug data set and attain promising results.

%\begin{itemize}
%\item Here's an interesting optimization problem from statistics
%  literature (that combined with knockoff is better than lasso?)
%\item Contribution: We show, remarkably, that this non-convex problem
%  can be globally optimized
%\item Significance:
%  \begin{itemize}
%  \item Rare ability to assess success of a non-convex method for
%    optimizing its criterion.
%  \item Opens large body of convex algorithms that can be applied to this problem
%  \item Allows to use global optimum in theoretical analysis
%  \end{itemize}
%\end{itemize}

\section{Main Proposal}
\label{sec:main-proposal}

\subsection{Reduction of TREX Problem to \texorpdfstring{$2p$}~ Convex Problems}
\label{sec:reduction-2p}

It is clear that the main complication with \eqref{eq:trex} and the
source of the non-convexity is the quantity 
$\|X^\top(Y-X\beta)\|_\infty$ in the denominator of the first term.
Observe that we may rewrite \eqref{eq:trex} as follows:
\begin{align*}
  P^*&:=\min_{\beta\in\real^p}\left\{\frac{\|Y-X\beta\|^2_2}{\max_{j\in\{1,\ldots,p\}}\consttrex|x_j^\top(Y-X\beta)|}+\|\beta\|_1\right\}\\
&=\min_{\beta\in\real^p}\min_{j\in\{1,\ldots,p\}}\left\{\frac{\|Y-X\beta\|^2_2}{\consttrex|x_j^\top(Y-X\beta)|}+\|\beta\|_1\right\}.
\end{align*}
The equality above shows that our problem can be viewed as the minimization of a pointwise minimum of $p$
functions.  Such a minimization problem can be alternately expressed
as finding the smallest of the $p$ functions' minima:
$$
P^*=\min_{j\in\{1,\ldots,p\}}P^*_j,
$$
where
$$
 P^*_j=\min_{\beta\in\real^p}\left\{\frac{\|Y-X\beta\|^2_2}{\consttrex|x_j^\top(Y-X\beta)|}+\|\beta\|_1\right\}.
$$
While the above problem is still non-convex, we show that its solution is obtainable by solving two convex optimization problems.  We use the simple fact that if $H_1$ and $H_2$ are subsets of $\real^p$ with $H_1\cup H_2=\real^p$, then
$$
\min_{\beta\in\real^p}g(\beta)=\min\left\{\min_{\beta\in H_1}g(\beta),\min_{\beta\in H_2}g(\beta)\right\}.
$$
Letting $H_1=\{\beta\in\real^p:x_j^\top(Y-X\beta)\ge0\}$ and $H_2=\{\beta\in\real^p:-x_j^\top(Y-X\beta)\ge0\}$, we may write $P_j^*$ as the minimum of two separate minimization problems:
$$
\min_{\beta\in\real^p}\left\{\frac{\|Y-X\beta\|^2_2}{\consttrex x_j^\top(Y-X\beta)}+\|\beta\|_1\st
 x_j^\top(Y-X\beta)\ge0\right\}
$$
and
\small
$$
\min_{\beta\in\real^p}\left\{\frac{\|Y-X\beta\|^2_2}{-\consttrex x_j^\top(Y-X\beta)}+\|\beta\|_1\st
 -x_j^\top(Y-X\beta)\ge0\right\}.
$$
\normalsize
In the above, we have used the fact that $|a|=a$ if $a\ge0$ and $|a|=-a$ if $a\le 0$. 
Both of these problems are of a common form, which can be expressed
in terms of a nonzero vector $v\in\real^n$ as
\small
\begin{align}
P^*(v):= \min_{\beta\in\real^p}\left\{\frac{\|Y-X\beta\|^2_2}{v^\top(Y-X\beta)}+\|\beta\|_1\st
  v^\top(Y-X\beta)\ge0\right\}.
\label{eq:convex}
\end{align}
\normalsize

Since the minimizer of $P_j^*$ must occur in one of these two half-spaces, we have that $P^*_j=\min\{P^*(\consttrex x_j),~P^*(-\consttrex x_j)\},$
and thus, we have shown in this section that 
$$
P^* = \min_{\substack{j\in\{1,\ldots,p\}\\s\in\{\pm 1\}}}P^*(s\consttrex x_j).
$$
The minimizer of \eqref{eq:trex} is therefore provided by the $(j,s)$ pair that attains the above minimization.  See Algorithm \ref{alg:ctrex} for the main algorithm, which we refer to as the {\em c-TREX} (short for convex-TREX). 
\begin{algorithm}[tb]
   \caption{The c-TREX algorithm for globally optimizing the TREX problem \eqref{eq:trex}.}
   \label{alg:ctrex}
\begin{algorithmic}
   \FOR{$j=1$ {\bfseries to} $p$}
   \FOR{$s\in\{-1,~1\}$}
   \STATE Solve the SOCP \eqref{eq:convex} with $v=s\consttrex x_j$ as described in Section \ref{sec:how-solve-each}.
   \STATE Let $\hat\beta(j,s)$ and $P^*(s\consttrex x_j)$ denote the optimal point and value, respectively.
   \ENDFOR
   \ENDFOR
   \STATE Let $(\hat j, \hat s)=\arg\min_{(j,s)} P^*(s\consttrex x_j)$
   \STATE Return $\hat\beta(\hat j, \hat s)$
\end{algorithmic}
\end{algorithm}
In the next section, we show that \eqref{eq:convex} is a convex
optimization problem that can be readily solved, which therefore
implies that we can globally minimize \eqref{eq:trex} by solving $2p$ convex problems.  The left panel of Figure~\ref{fig:convex-contours} shows the contours of the TREX objective in an example where $p=2$.  The right panel of Figure~\ref{fig:convex-contours} shows the decomposition of this non-convex problem into $4$ (i.e., $2p$) separate convex optimization problems.  The lowest of these $4$ minima is the global minimum of \eqref{eq:trex}.

%\begin{figure}[t]
%    \centering
%    \includegraphics[width=0.2\linewidth]{figs/trex_contour}
%    \caption{Contours of the TREX problem \eqref{eq:trex} in an example with $p=2$.}
%    \label{fig:contour}
%\end{figure}

\begin{figure}[t]
    \centering
    \includegraphics[width=0.38\linewidth]{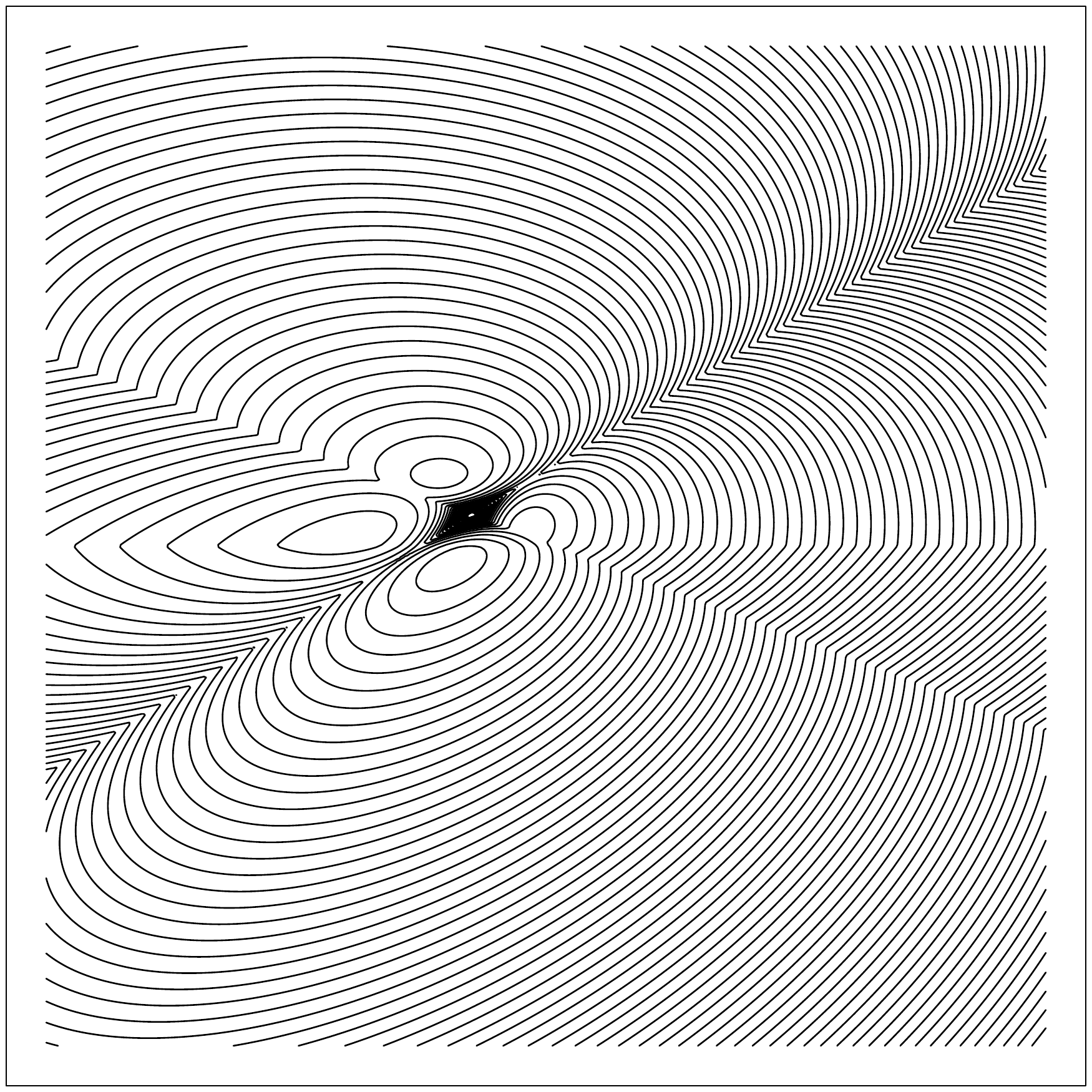}
    \includegraphics[width=0.4\linewidth]{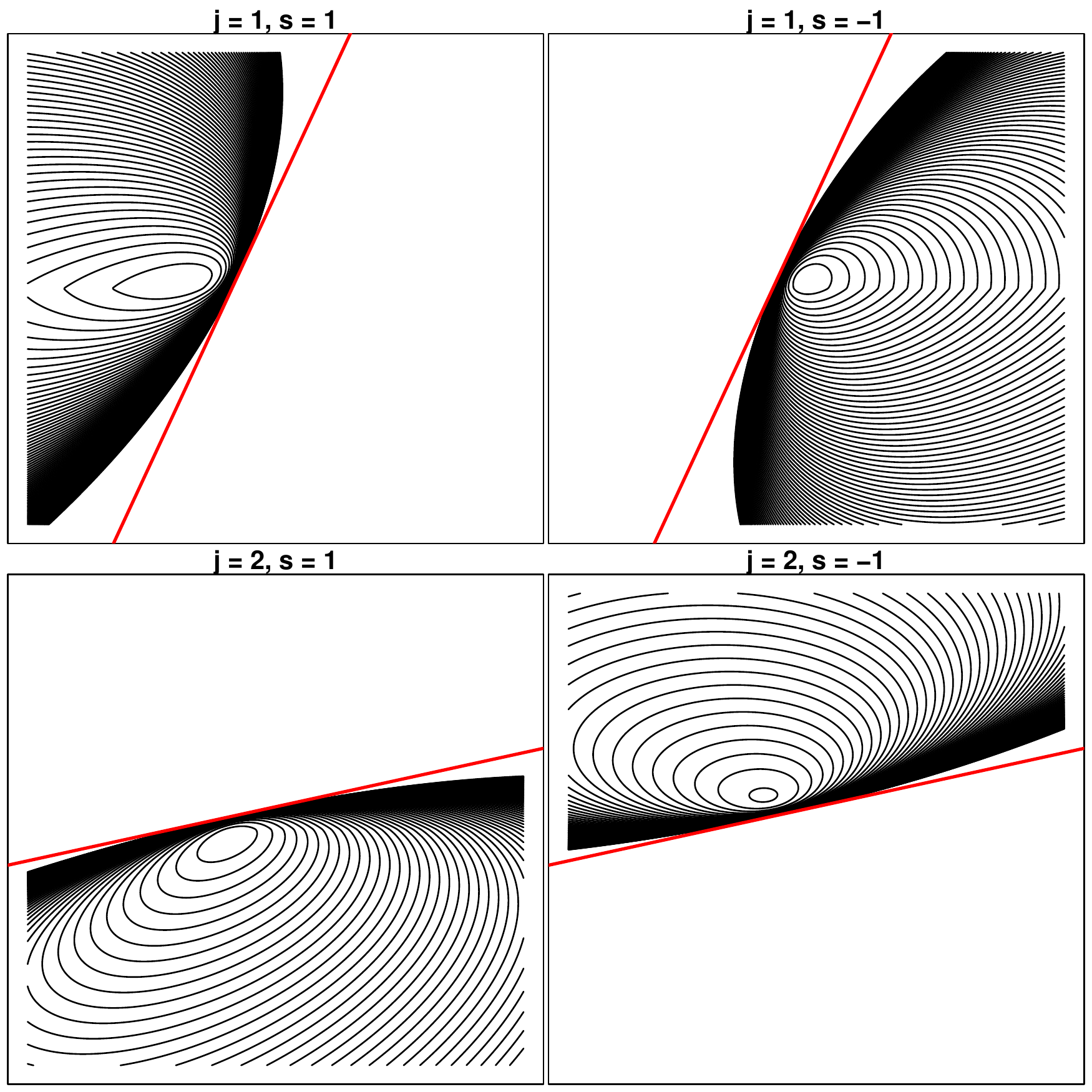}
    \caption{(Left) Contours of the TREX problem \eqref{eq:trex} in an example with $p=2$. (Right) The c-TREX, proposed in this paper, decomposes the TREX into $2p$ convex optimization problems, corresponding to $2p$ half-spaces of $\real^p$. The solution to TREX is the smallest of these $2p$ solutions.}
    \label{fig:convex-contours}
\end{figure}

\subsection{How to Solve Each Convex Problem}
\label{sec:how-solve-each}

The first term in \eqref{eq:convex} can be written as $f(Y-X\beta,v^\top(Y-X\beta))$ where $f(a,b)=a^2/b$, defined on $\real\times(0,\infty)$, is a fairly well-known convex function, sometimes referred to as ``quadratic-over-linear''
\citep{Boyd04}.  Since this term is the composition of a convex function and an affine function, it is therefore convex \citep{Boyd04}.  Following a technique used in \citet{Lobo98}, we can re-express \eqref{eq:convex} as a second-order cone program (SOCP).  We begin by writing \eqref{eq:convex} as
\begin{align*}
  \min_{t_0,\ldots,t_p}\sum_{j=0}^pt_j
  \quad\st &\|Y-X\beta\|^2_2\le t_0 v^\top(Y-X\beta)\\
  &v^\top(Y-X\beta)\ge0\\
&|\beta_j|\le t_j\text{ for }j=1,\ldots,p.
\end{align*}
A few lines of algebra give us a SOCP formulation of \eqref{eq:convex}:
\begin{align*}
  \min_{t\in\real^{p+1},\beta\in\real^p}&\sum_{j=0}^pt_j\\
  \st \Big\|&
    \begin{pmatrix}
2(Y-X\beta)\\v^\top(Y-X\beta)-t_0
    \end{pmatrix}\Big\|_2\le v^\top(Y-X\beta)+t_0\\
&\|e_j^\top\beta\|_2\le t_j\text{ for }j=1,\ldots,p,
\end{align*}
where $e_j\in\real^p$ denotes the $j$th canonical basis vector.
Writing \eqref{eq:convex} in this way not only exhibits it as a convex
optimization problem but also makes it clear how we can solve
\eqref{eq:convex} using existing SOCP solvers.  We consider two solvers in particular: ECOS (Embedded Conic Solver, \citealt{domahidi2013ecos}), an interior-point solver, and SCS (Splitting Conic Solver, \citealt{odonoghue2015conic}), a first-order solver.  In our experience, ECOS produces high-accurate solutions fairly rapidly for small- and mid-sized problems, but does not scale well for large problems.  By contrast, SCS can scale to much larger problem sizes by producing less accurate solutions. In practice, it is sometimes desirable to solve \eqref{eq:trex} along a grid of values of $\consttrex$. In such a case, SCS is convenient since it allows for warm-starting, which can greatly reduce the total amount of computational time. In particular, for each $\consttrex$, we maintain a set of $2p$ solutions to \eqref{eq:convex}, $\hat\beta(s\consttrex x_j)$ for $s\in\{\pm1\}$ and $j\in\{1,\ldots,p\}$.
Since a small modification to $\consttrex$ is not expected to make a big change to $\hat\beta(s \consttrex x_j)$ (for a fixed $s$ and $j$ pair), we can use $\hat\beta(s\tilde \consttrex x_j)$ for some $\tilde  \consttrex\approx \consttrex$ to initialize the solver to get $\hat\beta(s\consttrex x_j)$.

%\section{Additional Considerations}
%\label{sec:addit-cons}

%\begin{itemize}
%\item using cheap lower bounds to cut down on number of problems to
%  solve
%\item solving along a grid of c values:
%  \begin{itemize}
%  \item allows warm starts
%  \item if we know the set of cvx problems where objective was high at
%    one c, can we easily screen these out for a nearby c? (by
%    continuity of argmin as a function of c?)
%  \item lasso-style screening rules possible?
%  \end{itemize}
%\item topology of solutions 
%\end{itemize}

\section{Empirical Study of q-TREX and c-TREX}
\label{sec:empirical-qtrex}

\subsection{Investigating the Heuristic}
\label{sec:investigate}
While heuristic strategies are frequently used to attack non-convex optimization problems, it is rare that one is able to investigate the success of these heuristics.  In machine learning and statistics, it is common to evaluate the resulting predictions of the heuristic and to use that as ``evidence'' of success.  However, a method generating good predictions does not actually say anything about whether the heuristic is in fact successfully solving the original problem.  Another common form of ``evidence'' is for authors to rerun their heuristic with many random starts (leading to different local minima) and to show that most of the time it gets to the smallest observed one.  Again, this is not rigorous evidence of success since a method that consistently ends up in a sub-optimal local minimum will misleadingly look perfect.  A more principled approach that appears in, for example, the combinatorial optimization literature is to prove that the heuristic is guaranteed to get within some approximation ratio of the true solution.

There is therefore typically a disconnect between the motivating optimization problem and the method proposed in practice.  Since theoretical results are typically based on the original optimization problem rather than the heuristic, this disconnect leads to a gap between the ideal method that comes with theoretical guarantees and the practical method that is actually used.

The algorithm presented in this paper therefore presents us with a rare opportunity to investigate the performance of the q-TREX heuristic that was introduced in \citet{LedererMueller:14}.

\begin{figure}[!t] % was H
  %  \centering
  \begin{center}
  %\centerline{
   \includegraphics[width=0.32\linewidth]{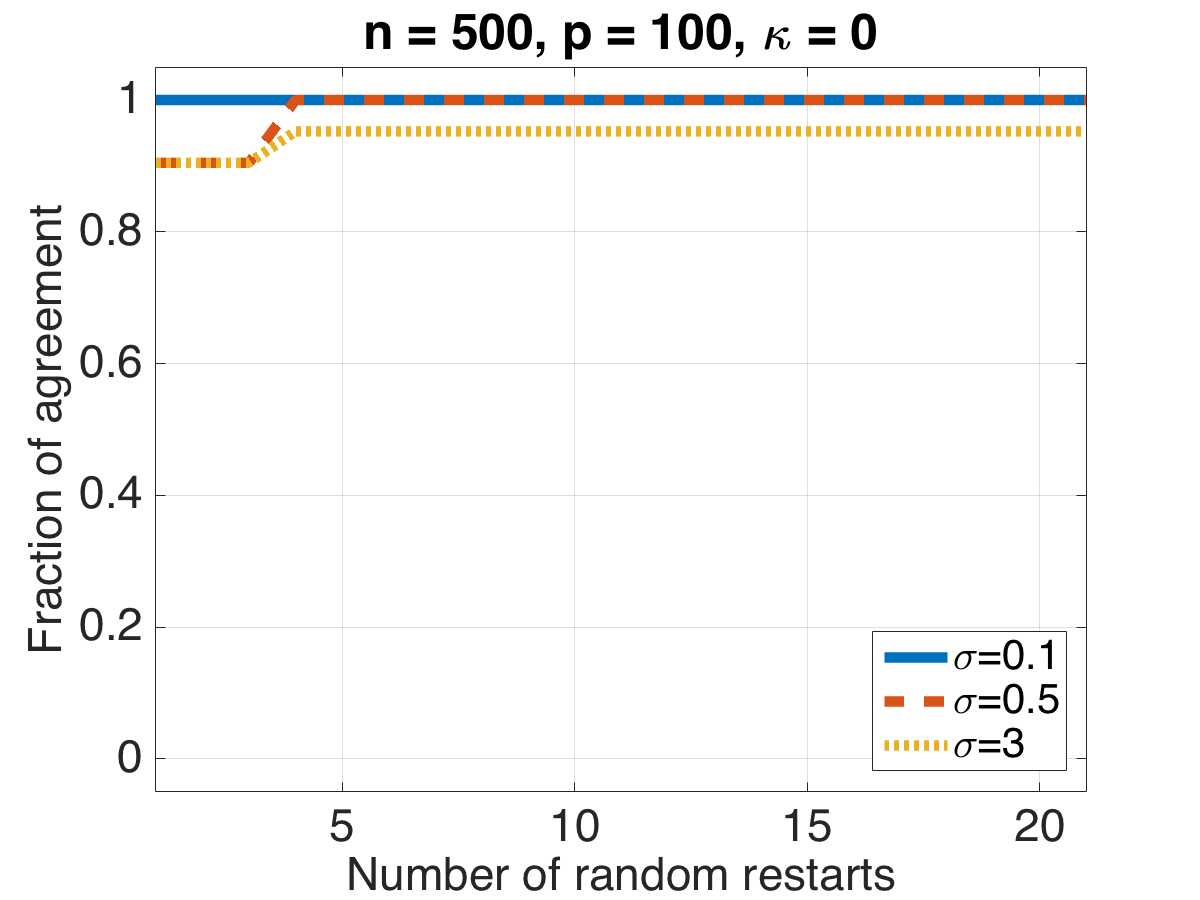}
  \includegraphics[width=0.32\linewidth]{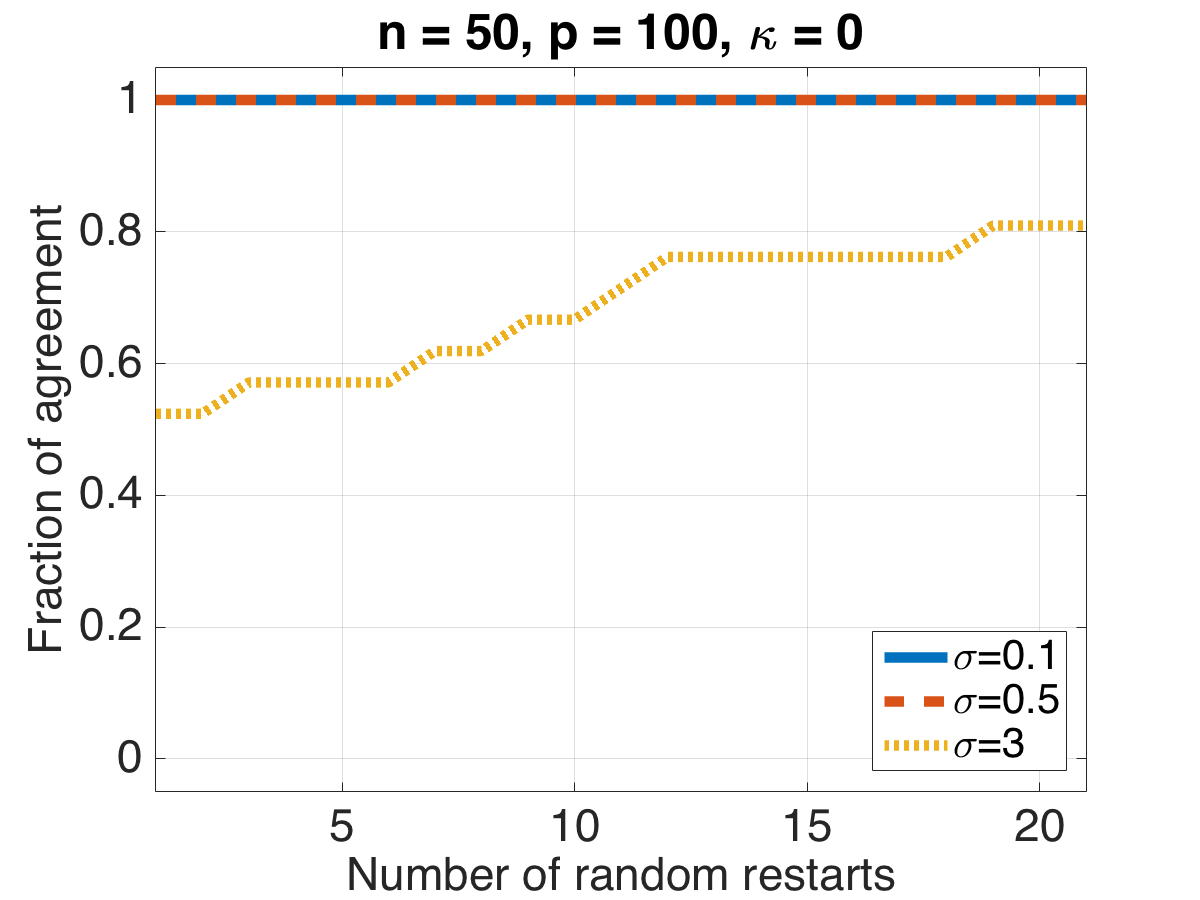}
  \includegraphics[width=0.32\linewidth]{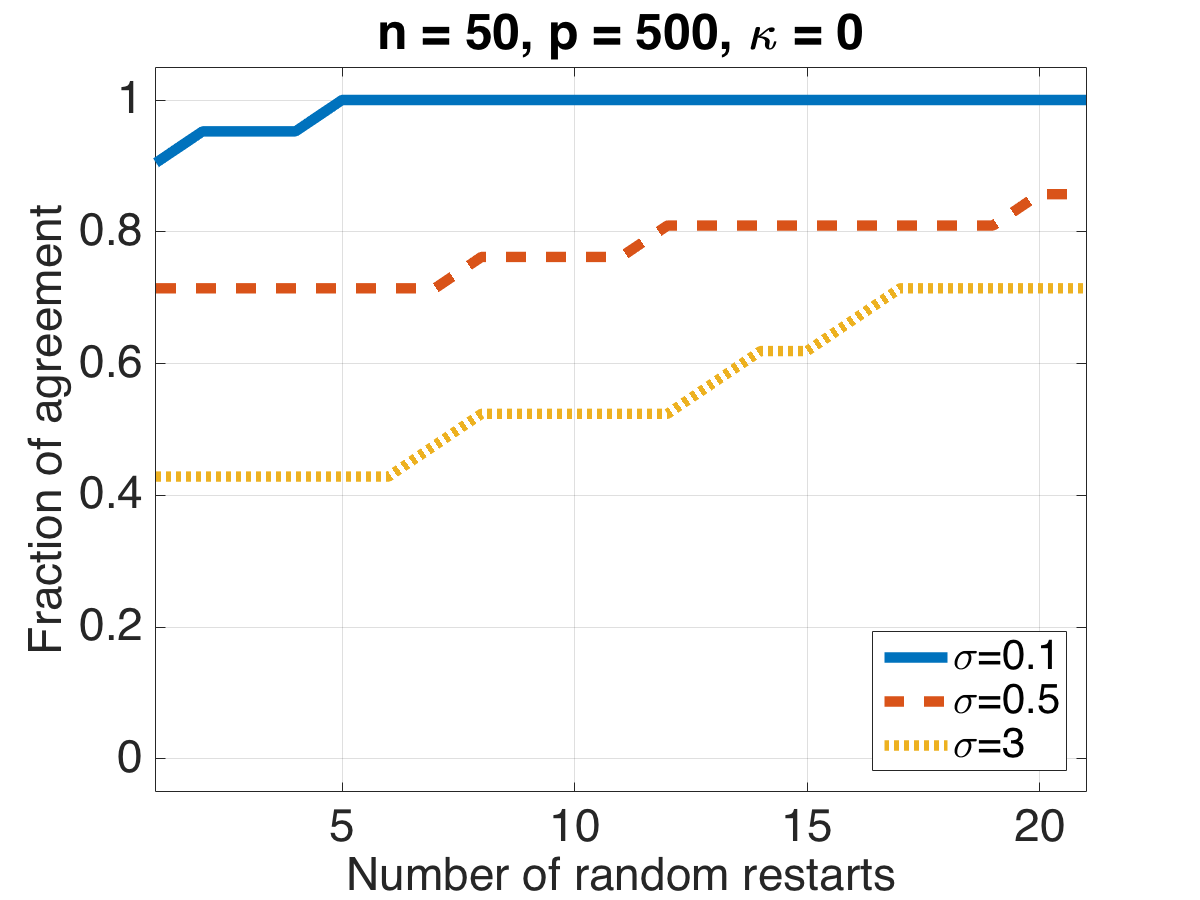}\\
    \includegraphics[width=0.32\linewidth]{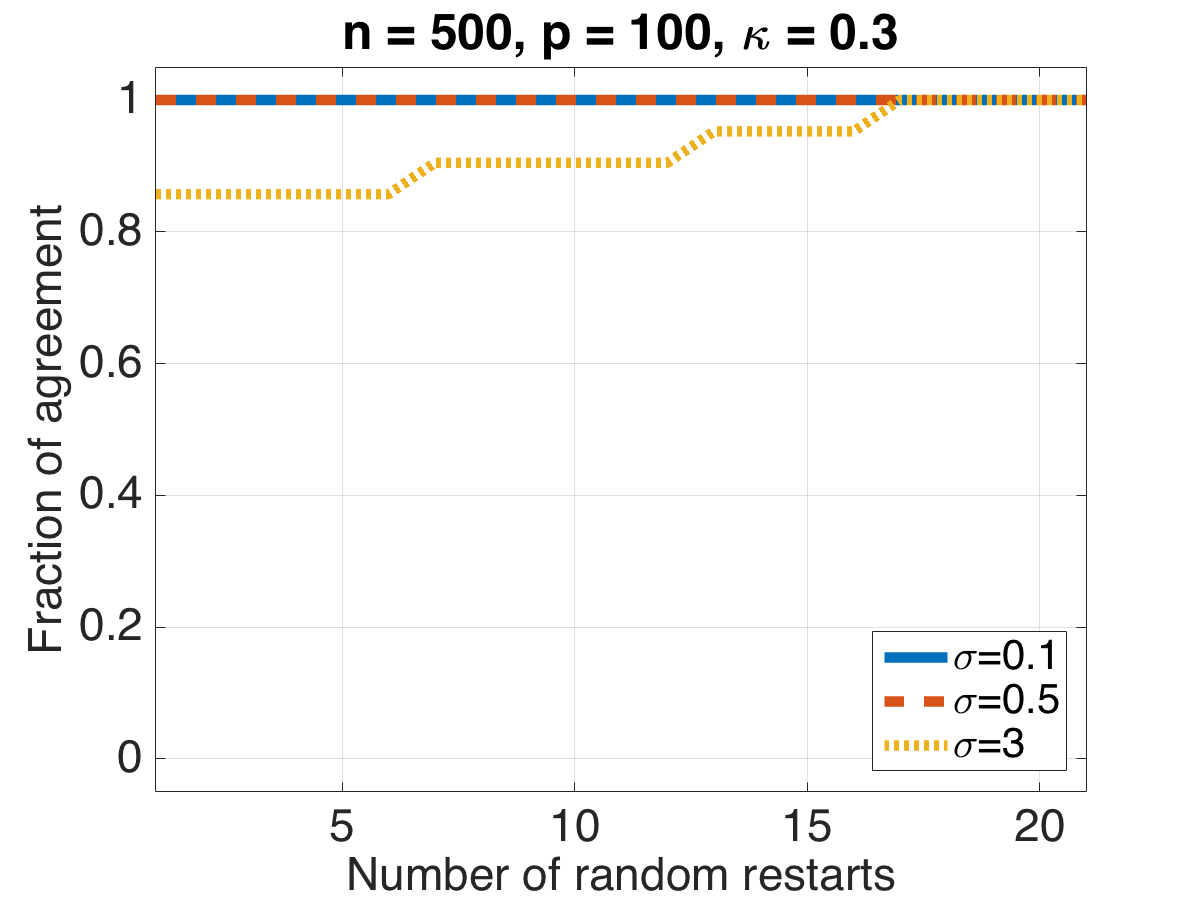}
  \includegraphics[width=0.32\linewidth]{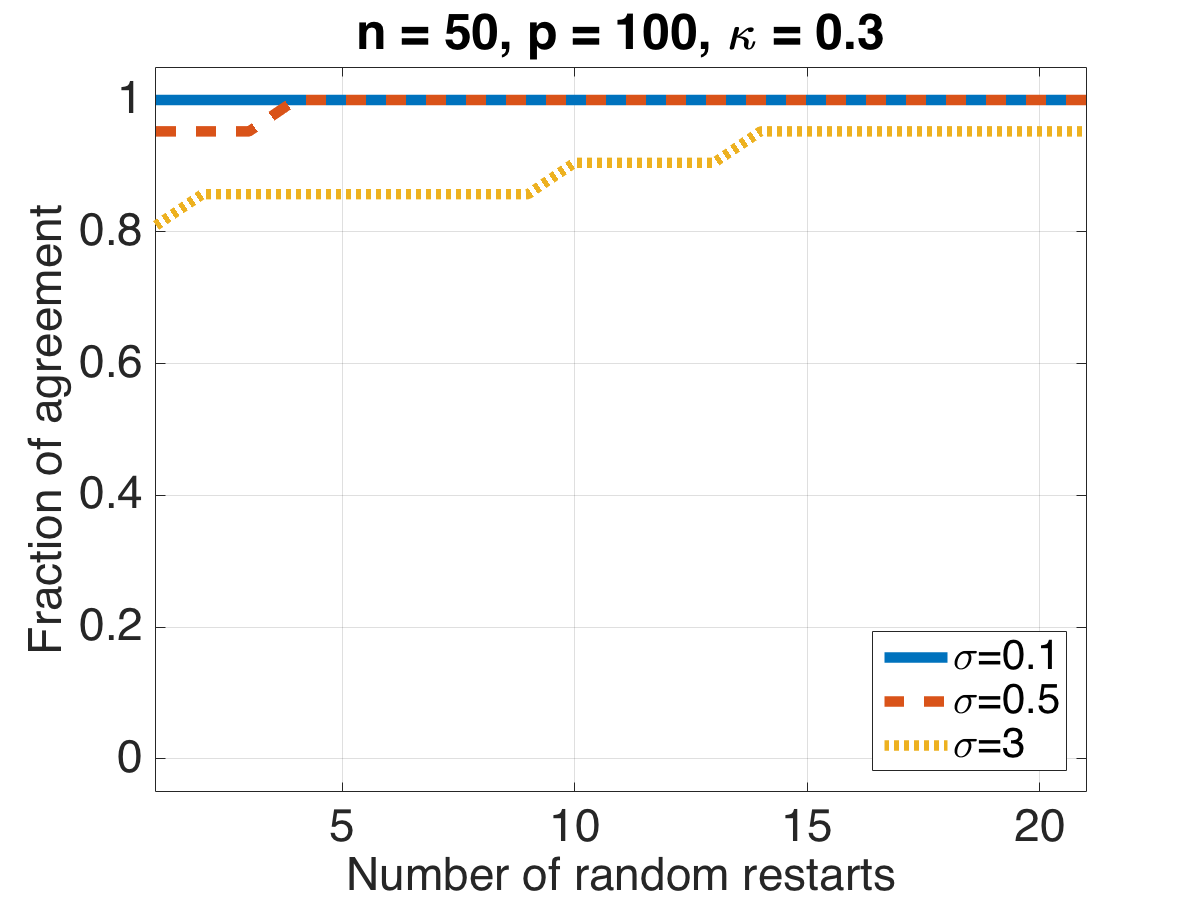}
  \includegraphics[width=0.32\linewidth]{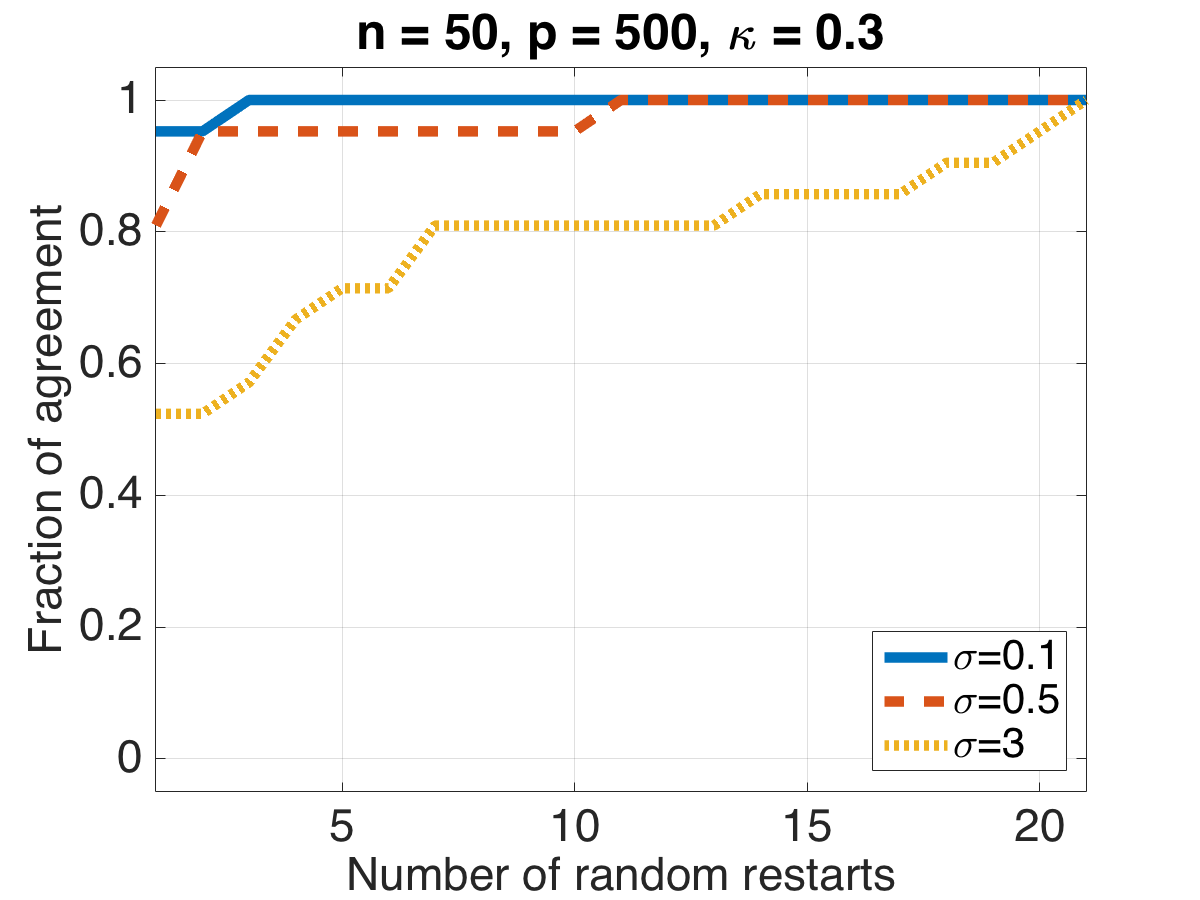}\\
    \includegraphics[width=0.32\linewidth]{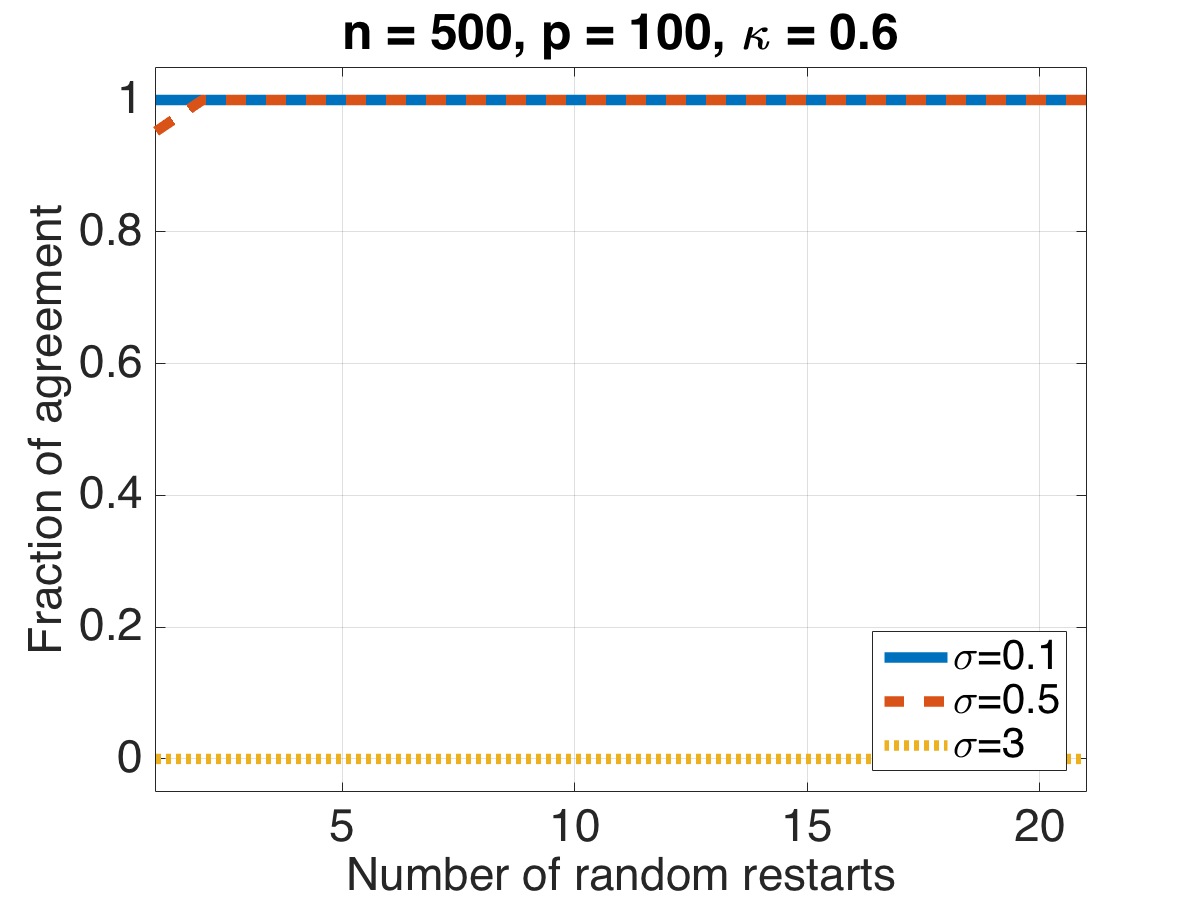}
  \includegraphics[width=0.32\linewidth]{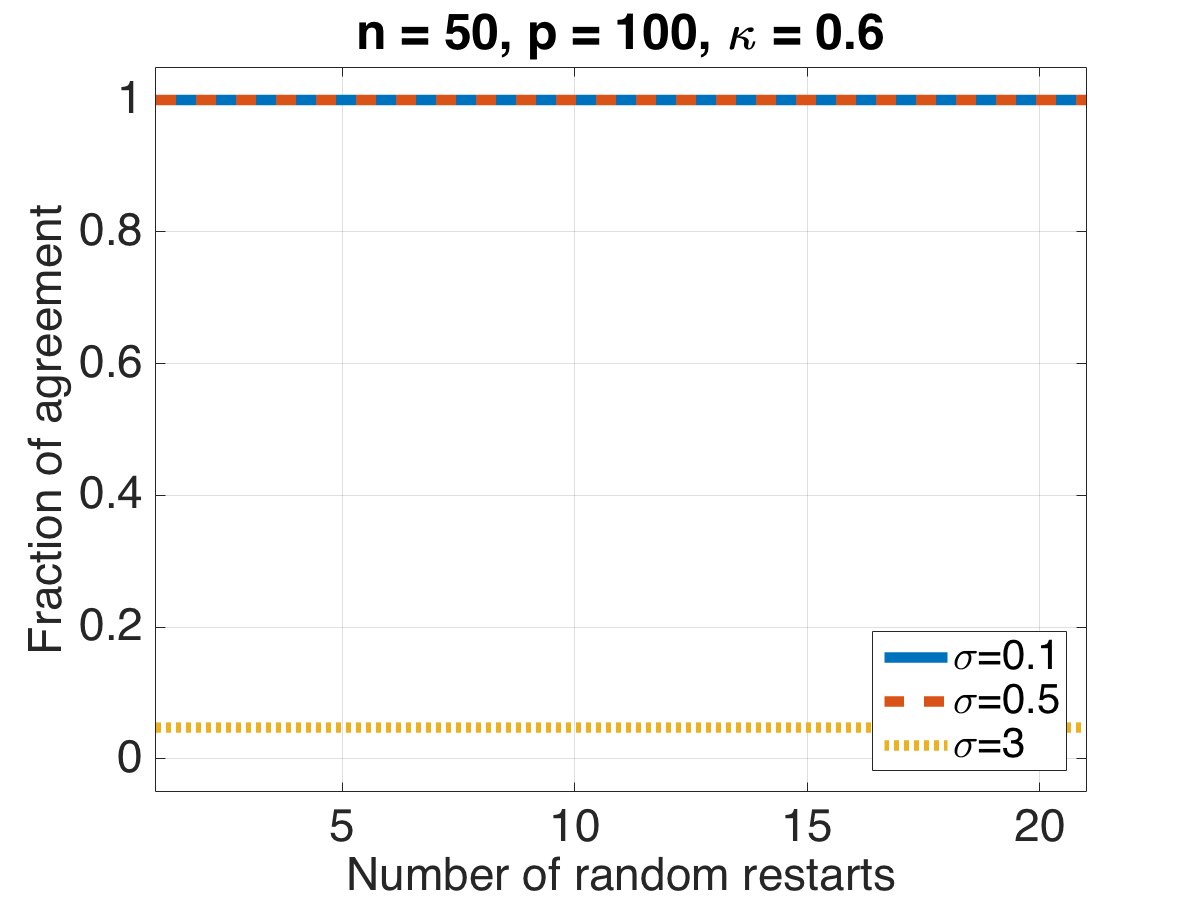}
  \includegraphics[width=0.32\linewidth]{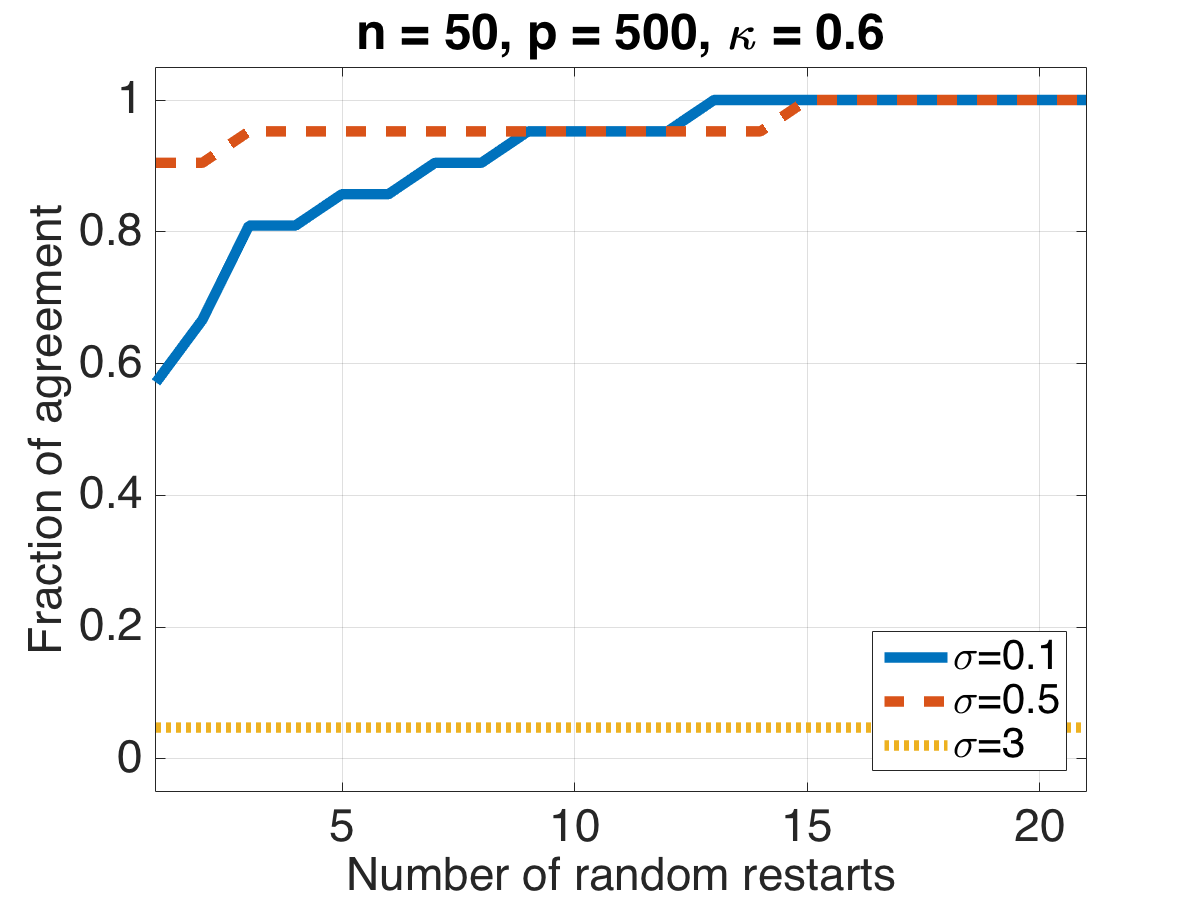}\\
    \includegraphics[width=0.32\linewidth]{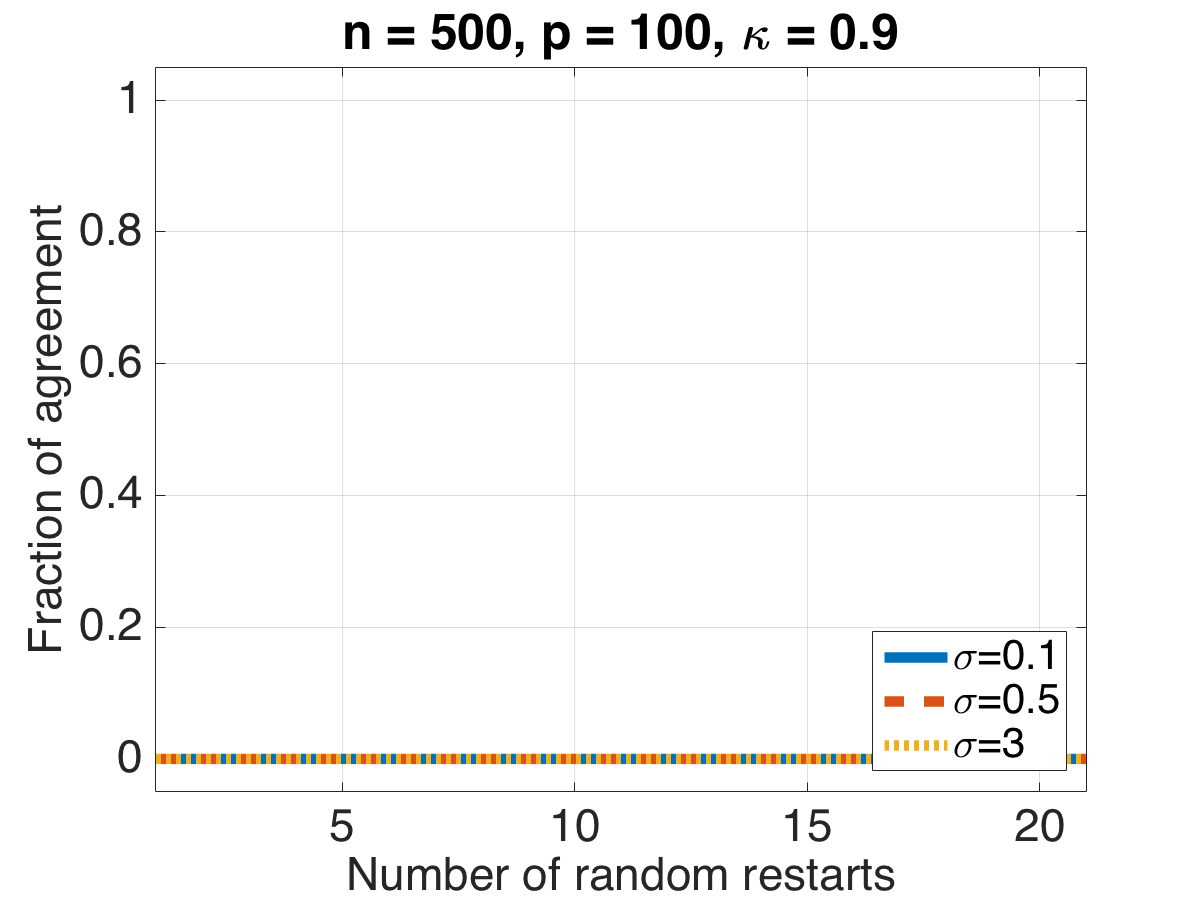}
  \includegraphics[width=0.32\linewidth]{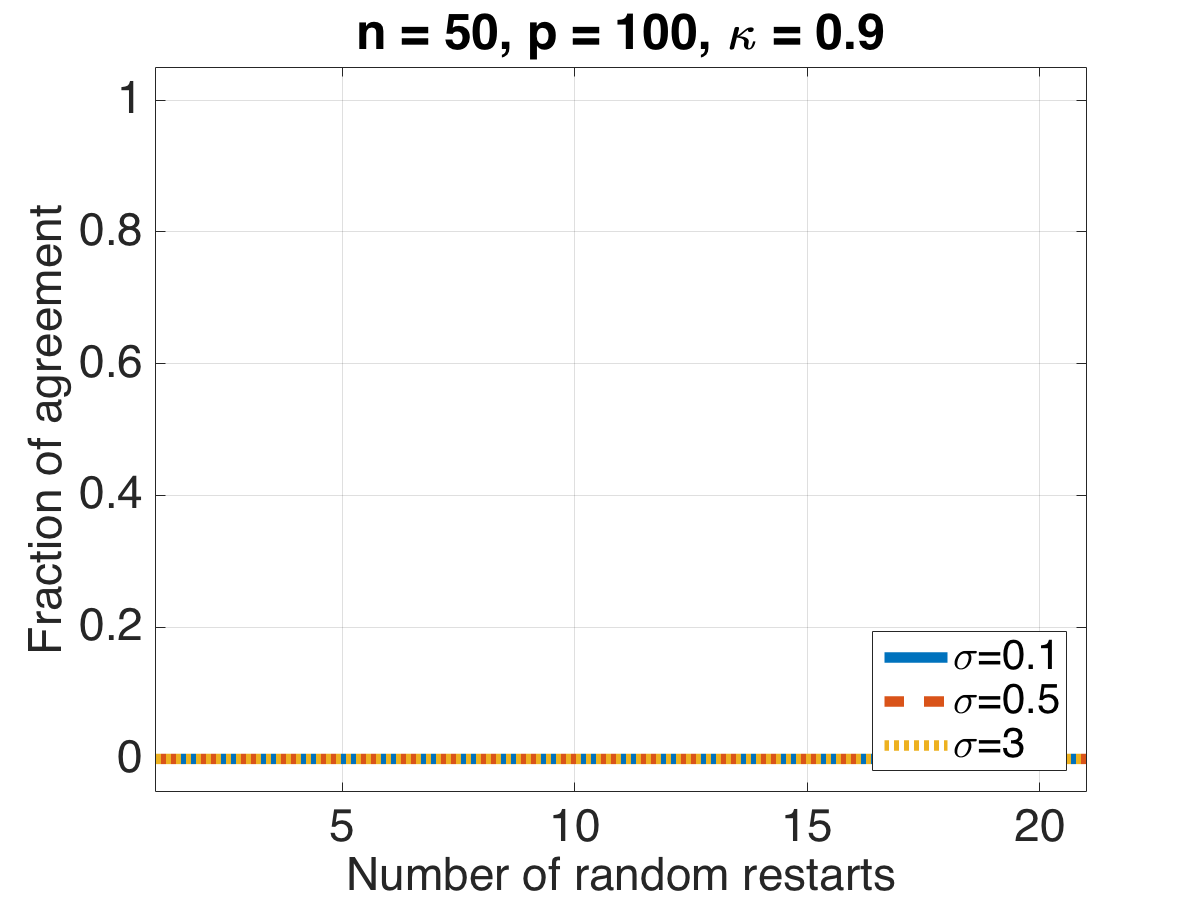}
  \includegraphics[width=0.32\linewidth]{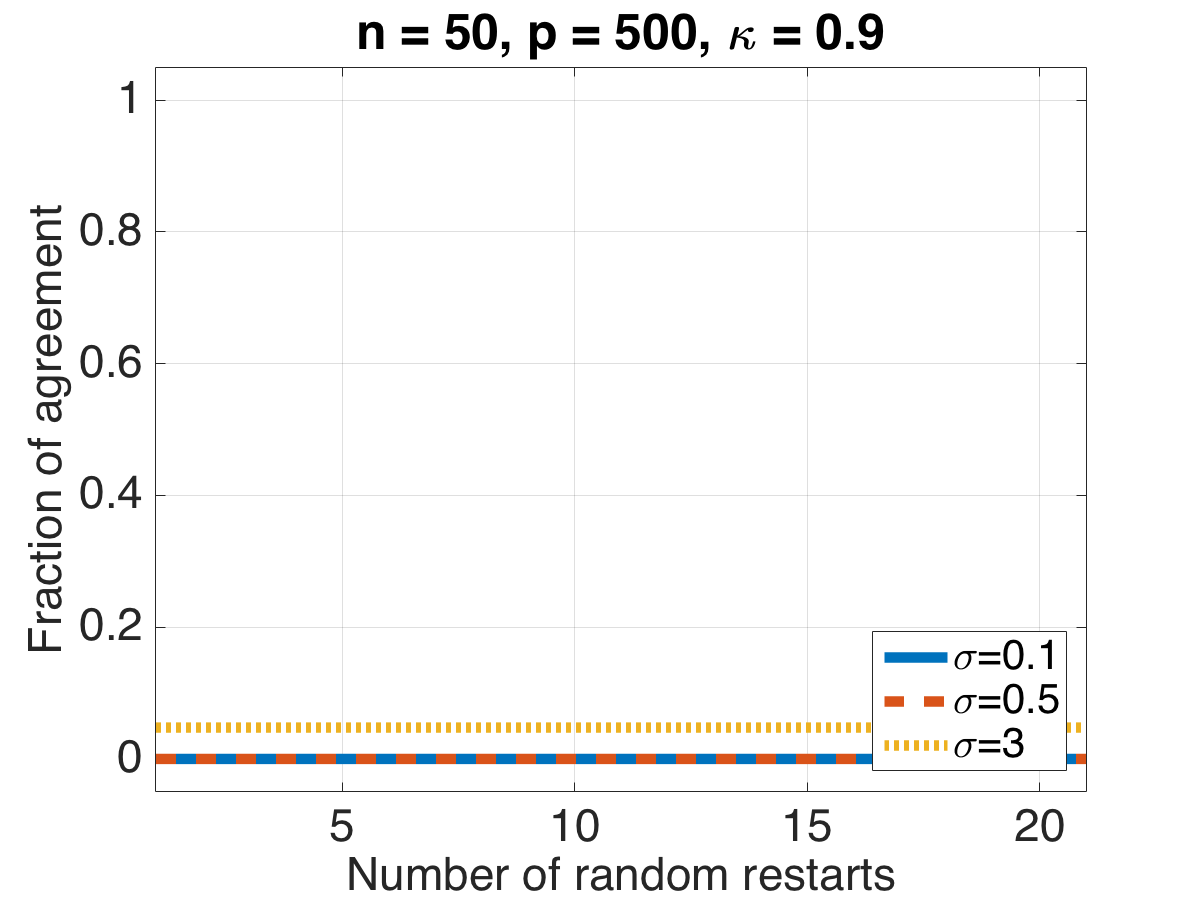}
  %}
  \caption{Probability that q-TREX gets within $10^{-4}$ of the global optimal value as a function of number of random restarts. (The first starting point is always taken to be the vector of all zeros.)}
  \label{fig:prob-success}
  \end{center}
\end{figure}

For the empirical study, we compare the performance of c-TREX and q-TREX using a simulation scenario similar to \citet{LedererMueller:14}. We generate data according to the linear model $Y_i=X_i^{\top}\beta+\varepsilon_i$, $i=1,\dots,n$, with three regimes for the sample size $n$ and the number of variables $p$, $(n,p)\in\{(500,100),(50,100),(50,500)\}$.The first regime corresponds to large sample setting $n>p$, and the other two correspond to low sample setting $n<p$. We set the number of nonzero variables $s=5$, regression vector $\beta=(1_s, 0_{p-s})$, errors $\varepsilon_i\sim N(0,\sigma^2)$ with $\sigma\in\{0.1,0.5,3\}$, and vector of predictors $X_i\sim N(0,\Sigma)$ with $\Sigma_{ii}=1$ and $\Sigma_{ij}=\kappa$ with $\kappa\in\{0,0.3,0.6,0.9\}$. We have also tried $\beta=(0.2,0.4,0.6,0.8,1.0,0_{p-s})$ and obtained the same qualitative results. We consider $n_{rep}=21$ replications for each combination of $\{p,\kappa,\sigma\}$. We use $n_{starts}=21$ initial values for $\beta$ for q-TREX with $\beta^{(0)}_1=0$ and $\beta^{(0)}_i$, $i=2,...,n_{starts}$ initialized at random with $25\%$ nonzero features. We have also tried initializing q-TREX with lasso solutions obtained via glmnet \citep{glmnet2013}, but the results are nearly identical to random initializations. For all the simulations, we set TREX constant $\consttrex=0.5$.

Figure~\ref{fig:prob-success} shows the empirical probability (over $n_{rep}=21$ replications) of q-TREX attaining an objective value within $10^{-4}$ of the global minimum as a function of number of restarts. The q-TREX is successful at recovering the global minimum as long as $\kappa$ and $\sigma$ are not too large. Specifically, q-TREX fails to recover global solution when $\kappa=0.9$ and consistently has low success probability when $\sigma=3$. As expected, increasing the number of initial starting points leads to a larger probability of success, however using only one starting point $\beta^{(0)}=0$ provides satisfactory performance for small $\kappa$ and $\sigma$.

\subsection{Timing Results}

We compare q-TREX and c-TREX timing performance on a laptop with 3.1 GHz Intel Core i7 using Matlab\_R2015b. The timing for both changes significantly with the dimension $p$, and is not significantly influenced by $\kappa$ or $\sigma$. In Figure~\ref{fig:runtime} we present results for $\kappa\in\{0,0.9\}$. The execution time reported for q-TREX  is the total time with $41$ restarts; the execution time reported for c-TREX is the total time over $2p$ problems using the ECOS solver~\citep{domahidi2013ecos}. The q-TREX is significantly faster than c-TREX.

\begin{figure}[!t] % was H
    \centering
    \includegraphics[width=0.32\linewidth]{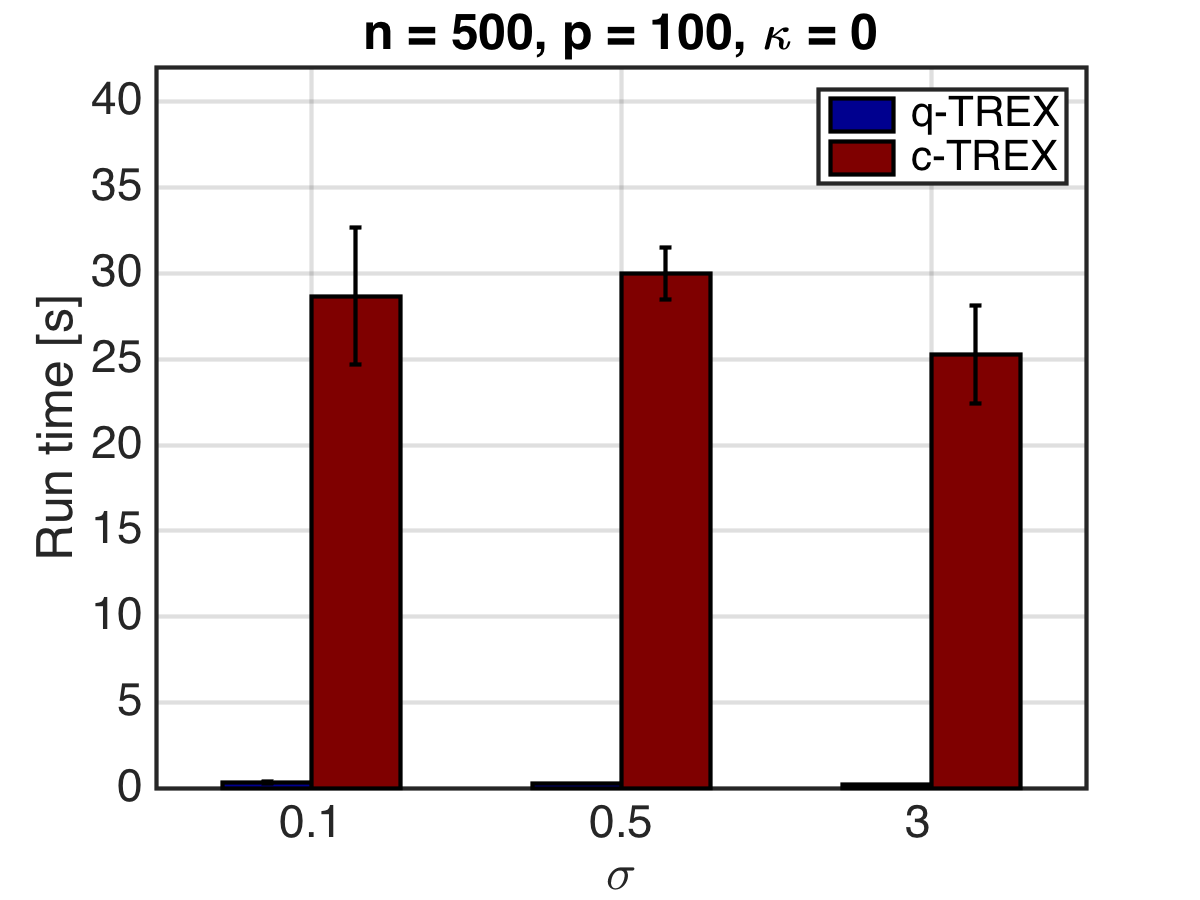}
  \includegraphics[width=0.32\linewidth]{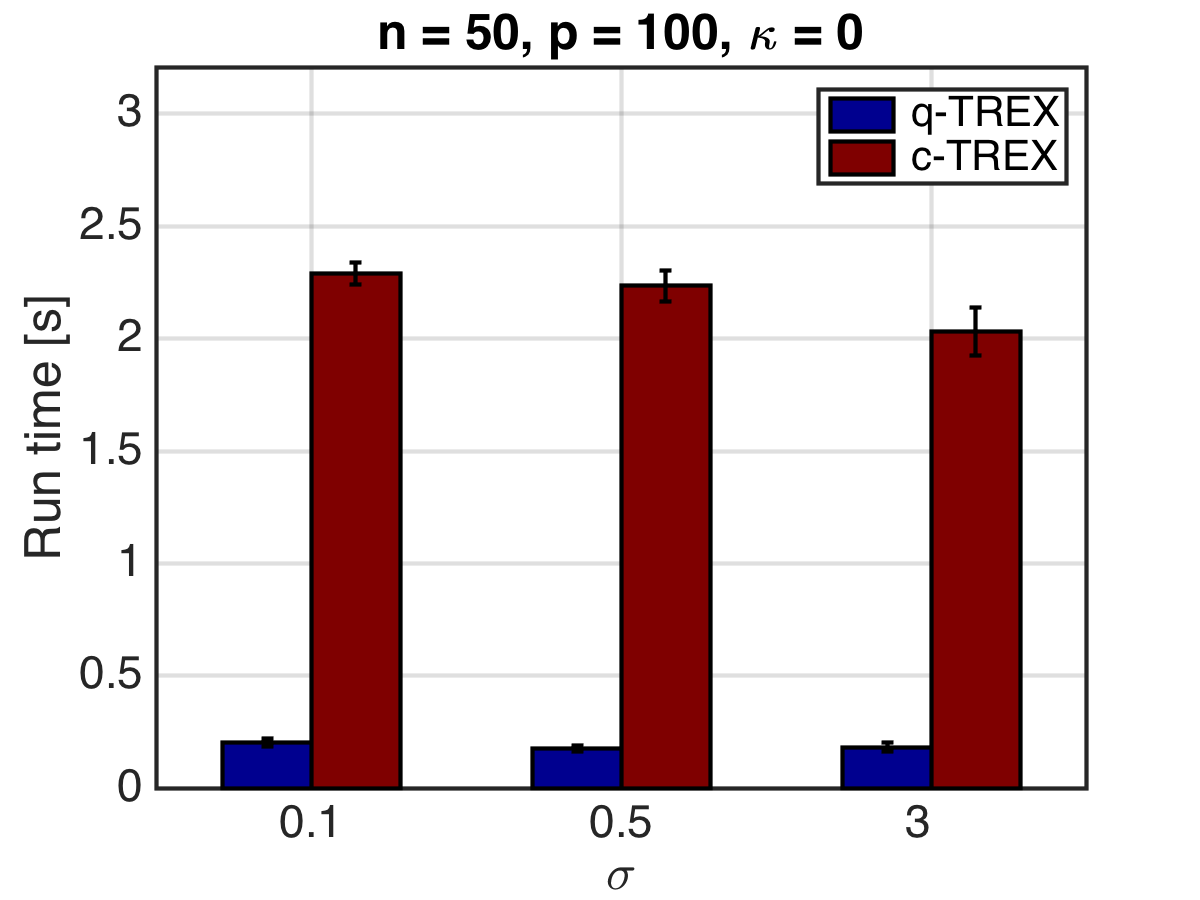}
  \includegraphics[width=0.32\linewidth]{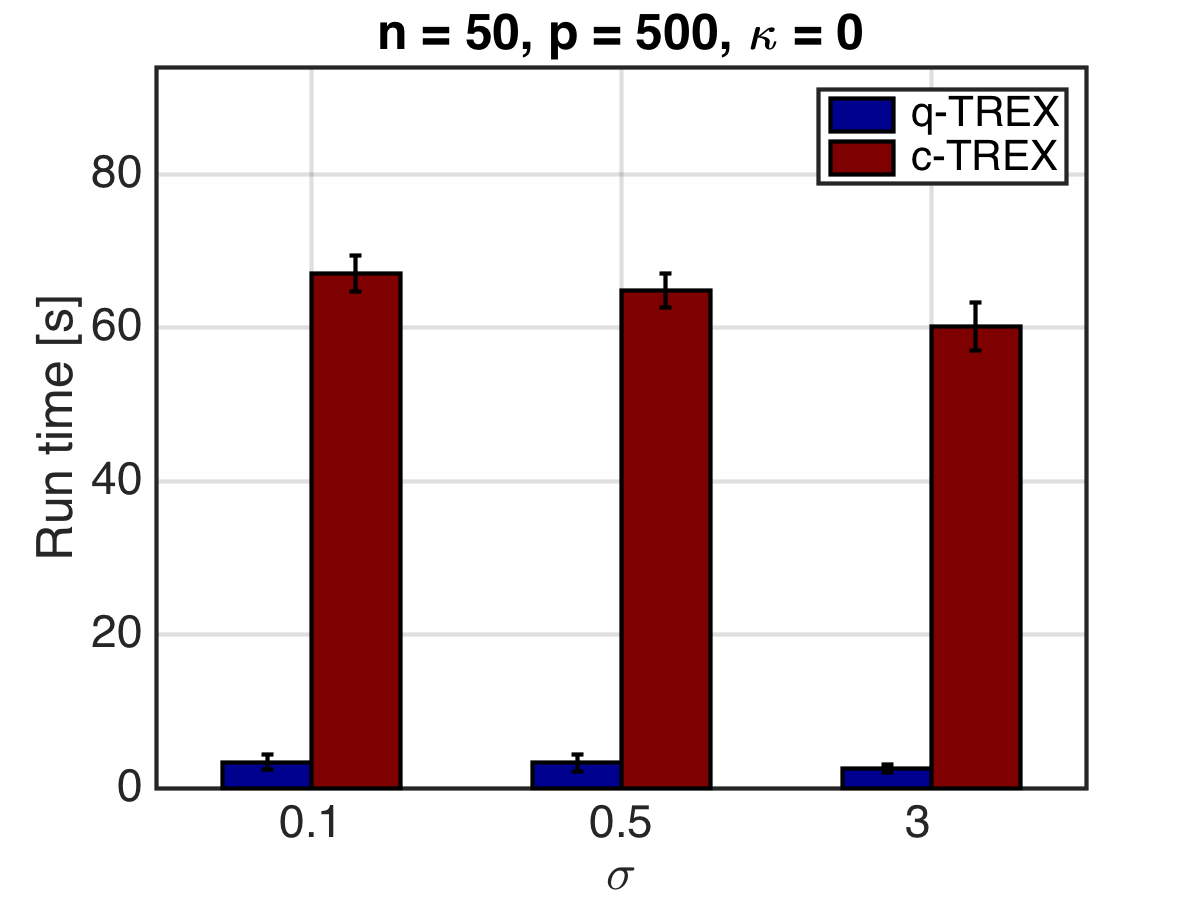}\\
  \includegraphics[width=0.32\linewidth]{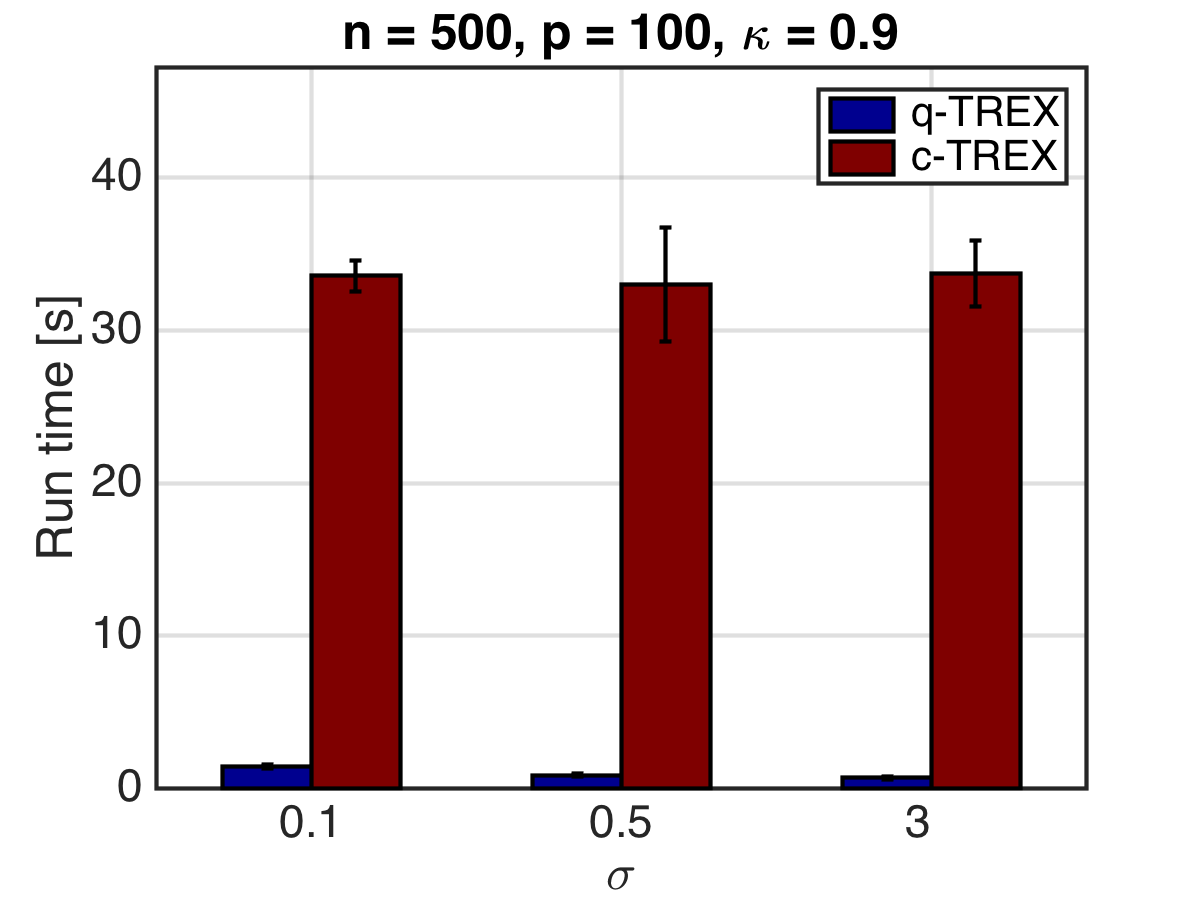}
  \includegraphics[width=0.32\linewidth]{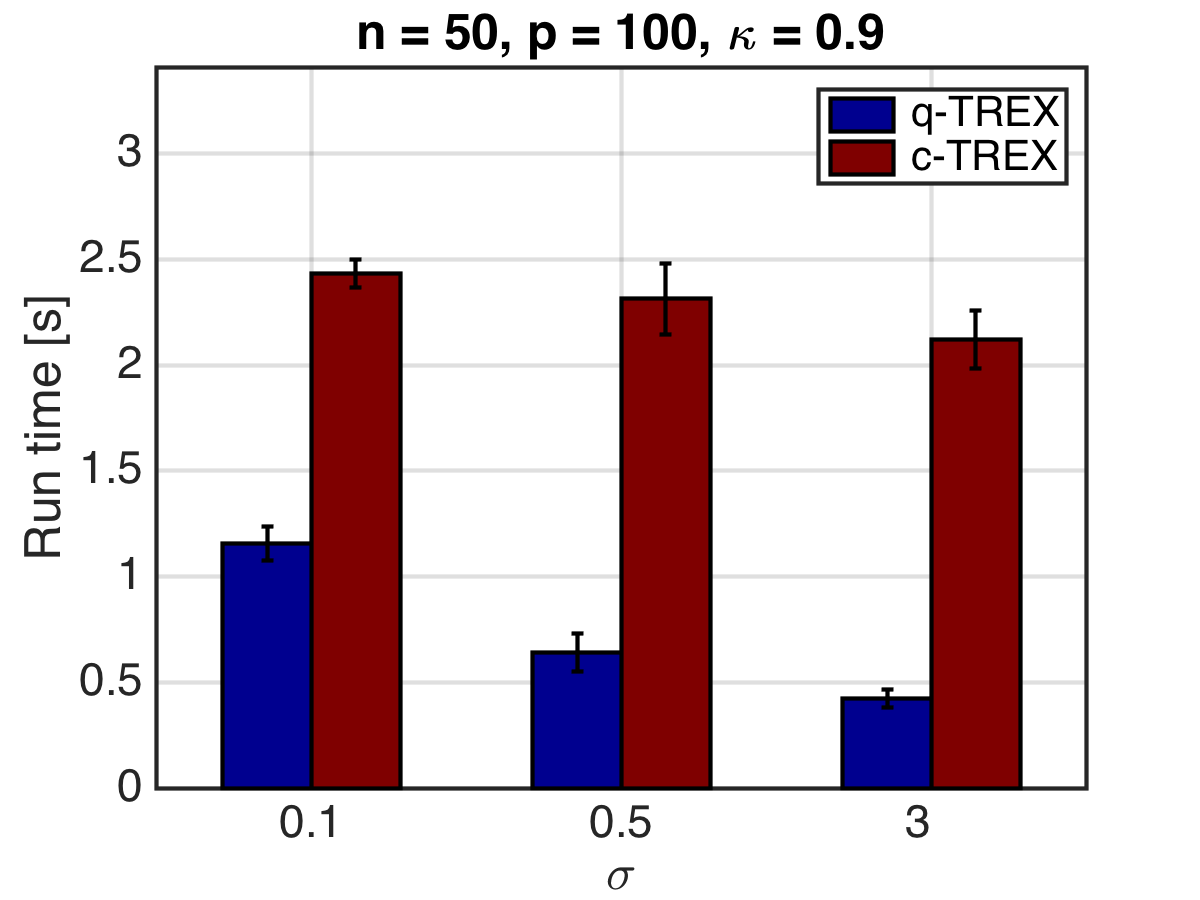}
  \includegraphics[width=0.32\linewidth]{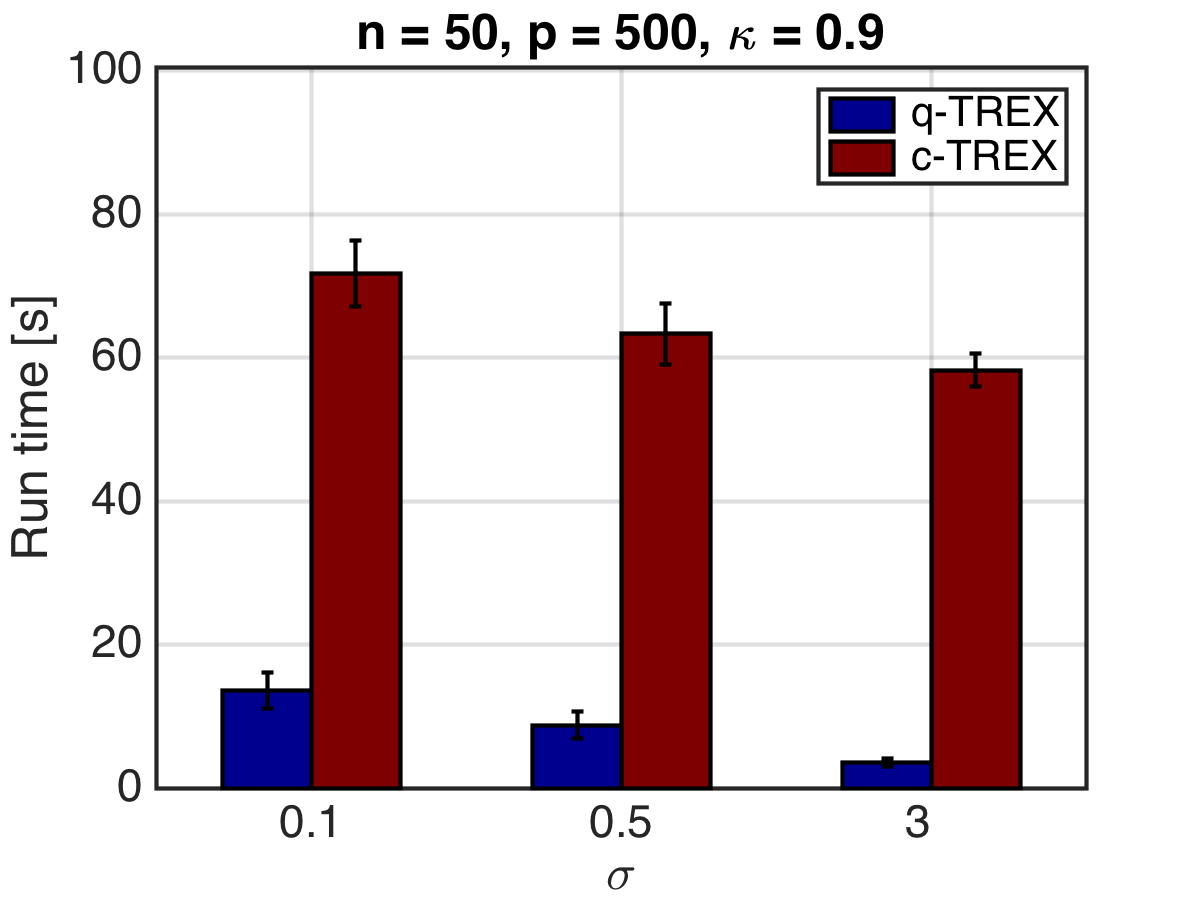}
  \caption{Average run time of q-TREX and c-TREX over $n_{rep}=21$ replications, $\kappa=0$ and $\kappa=0.9$.}
  \label{fig:runtime}
\end{figure}

\subsection{Statistical Performance}
\label{sec:statPerform}

We have seen that q-TREX is much faster than c-TREX; however, Section \ref{sec:investigate} shows that q-TREX fails to achieve the global minimization in some situations, for example when $\kappa=0.9$. Here we investigate whether this computational discrepancy has an effect on statistical performance. Specifically, we compare the estimation error $\|\hat \beta-\betatrue\|_2$ for q-TREX and c-TREX when $\kappa\in \{0,0.9\}$ (Figure~\ref{fig:estimation}, the results for $\kappa\in\{0.3,0.6\}$ are similar). The estimation error of q-TREX is on average the same as for c-TREX for all combinations of $\{p,\kappa,\sigma\}$. While we of course do not usually know the true values of $\kappa$ and $\sigma$ in real settings, we find no evidence in terms of estimator performance that one should prefer the exact TREX solution over the q-TREX solution: If $\kappa$ and $\sigma$ are both small, the two methods result in the same computational and statistical performance. If either $\kappa$ or $\sigma$ is large, q-TREX may fail to achieve the global optimal value, however this will not affect the statistical performance.

\begin{figure}[!t] %was H
    \centering
      \includegraphics[width=0.32\linewidth]{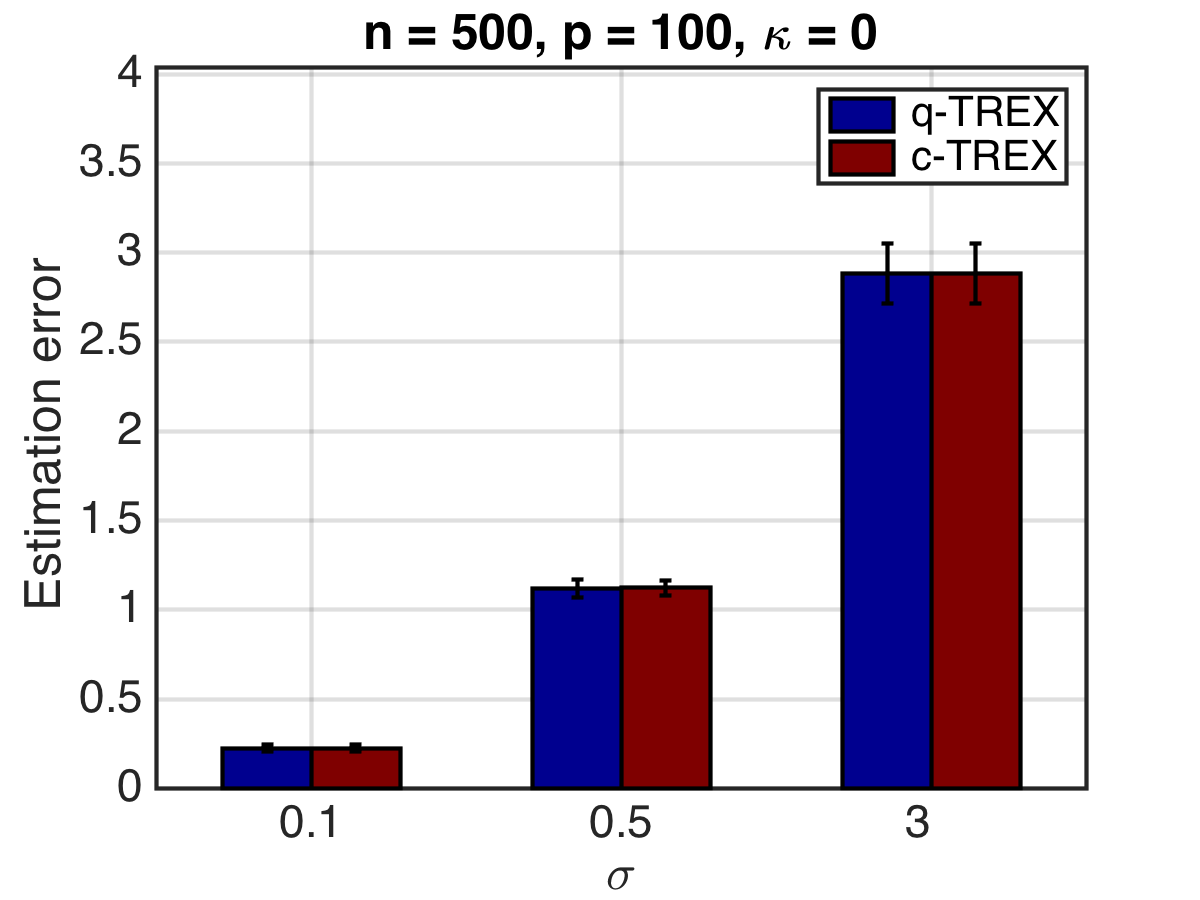}
  \includegraphics[width=0.32\linewidth]{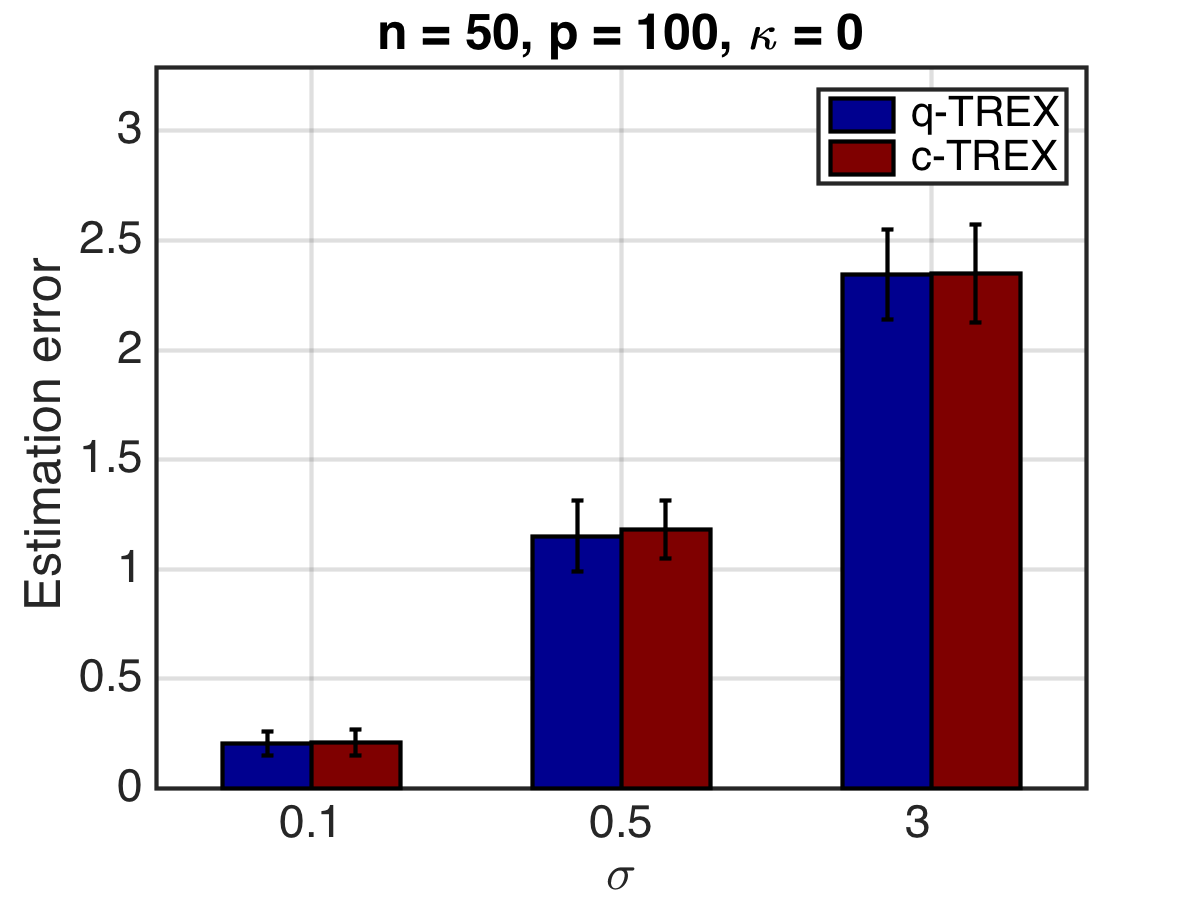}
  \includegraphics[width=0.32\linewidth]{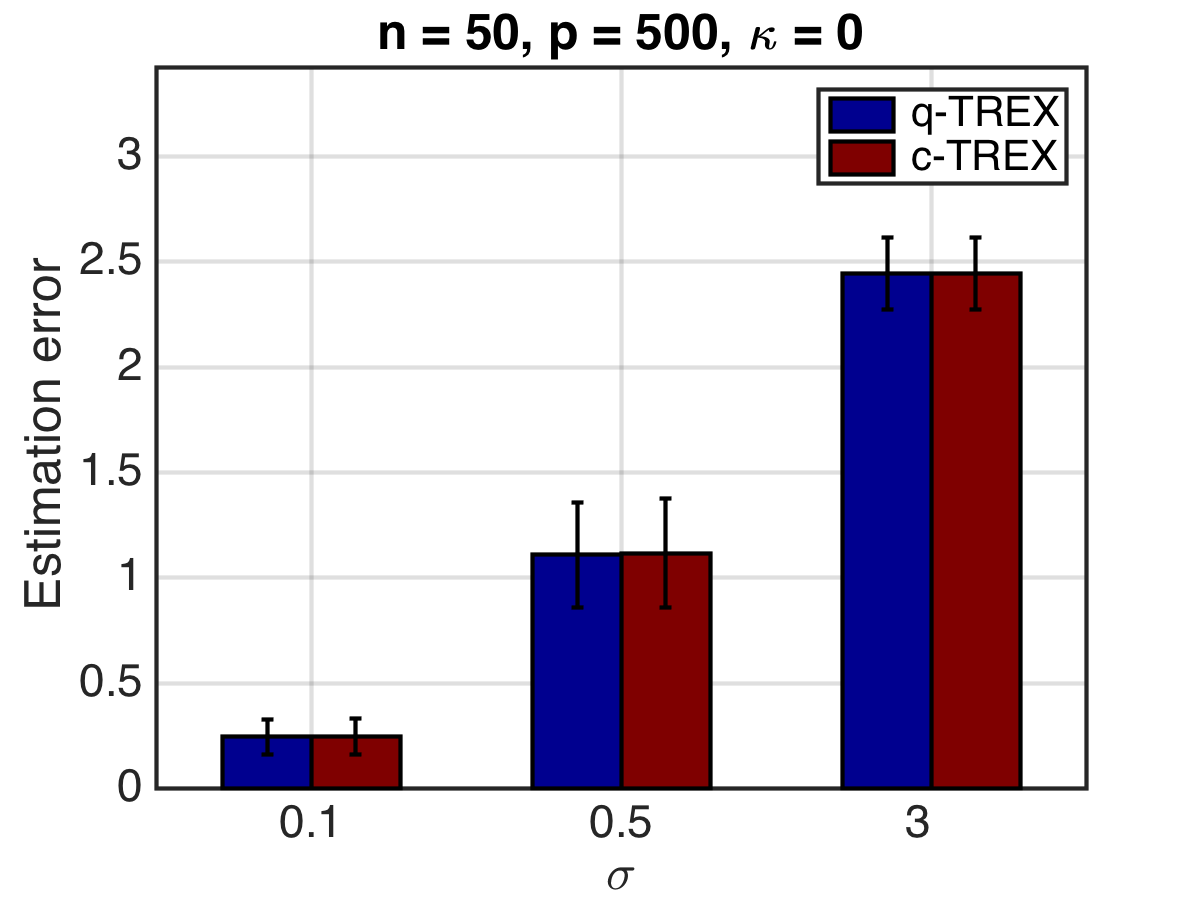}\\
    \includegraphics[width=0.32\linewidth]{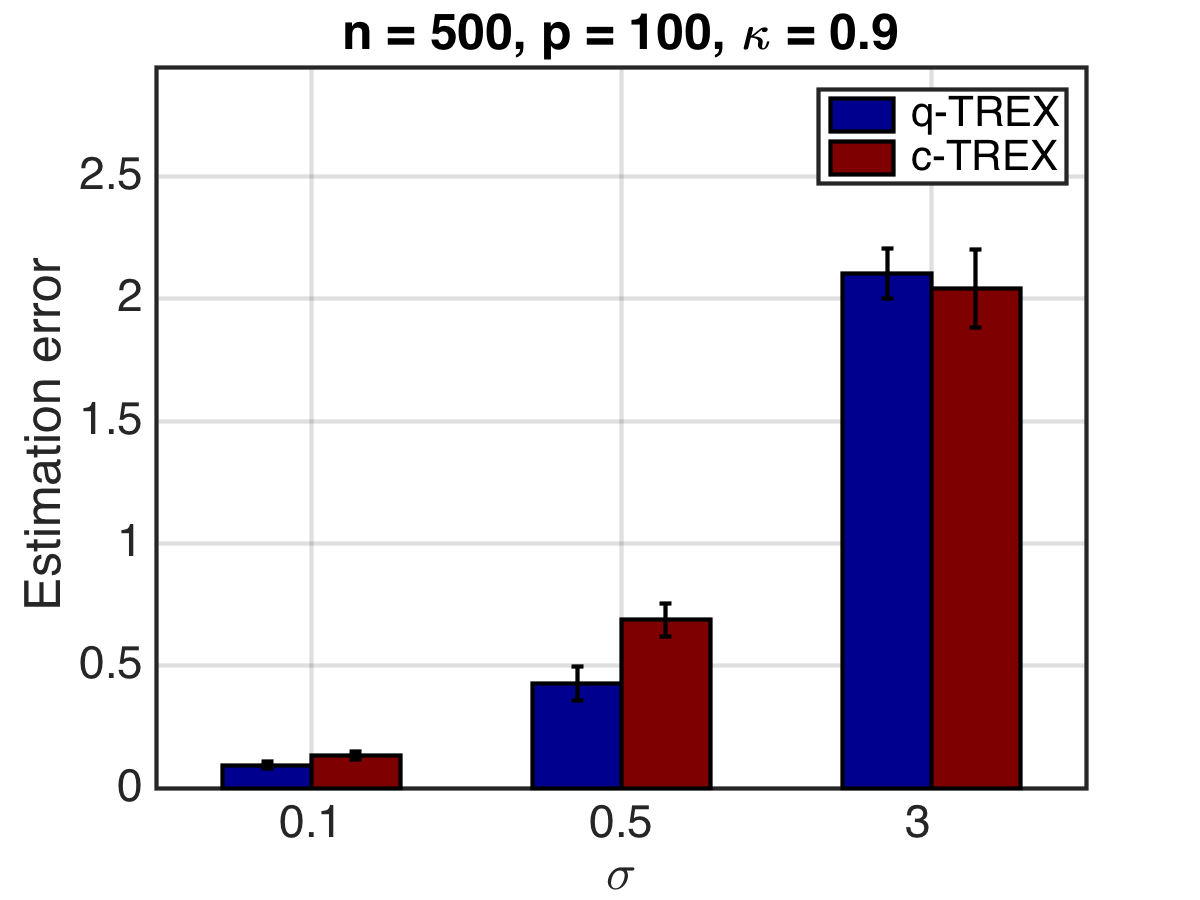}
  \includegraphics[width=0.32\linewidth]{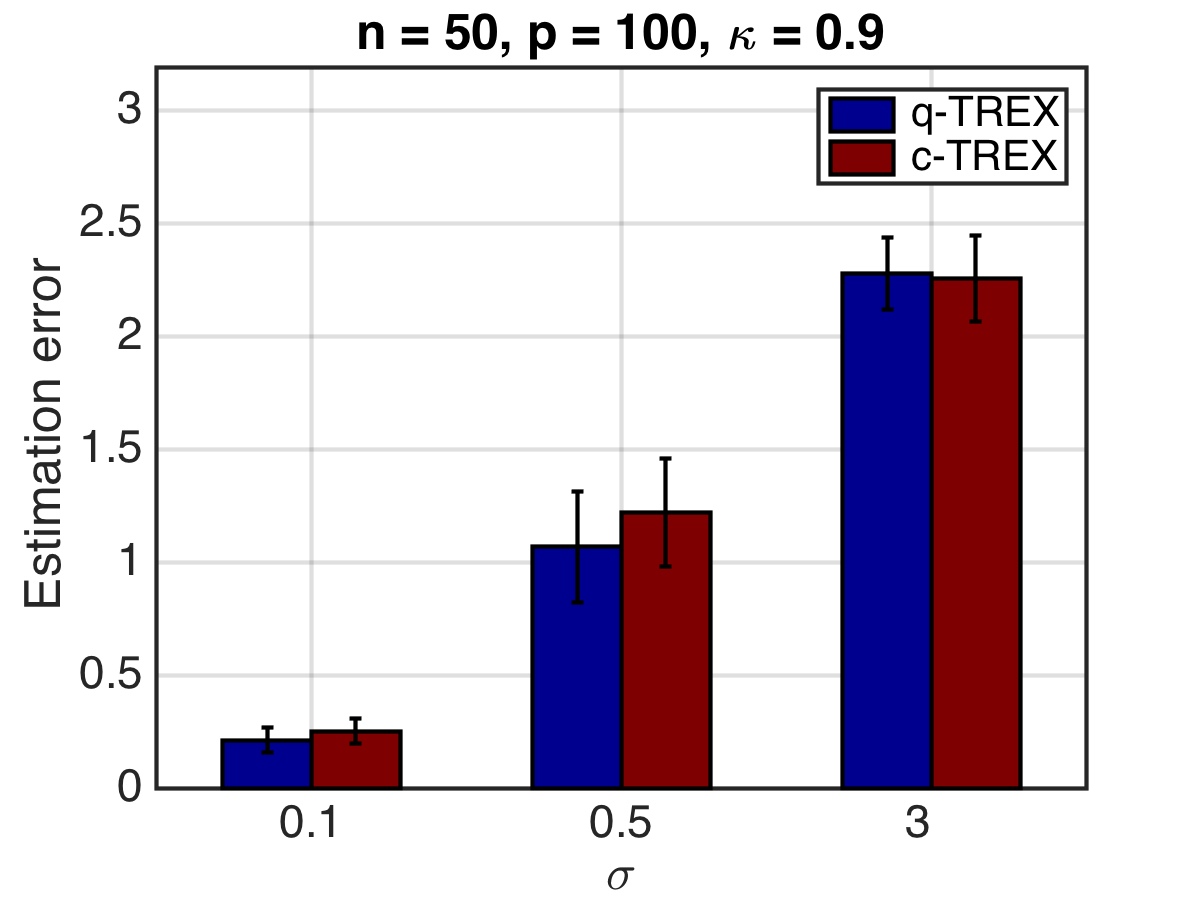}
  \includegraphics[width=0.32\linewidth]{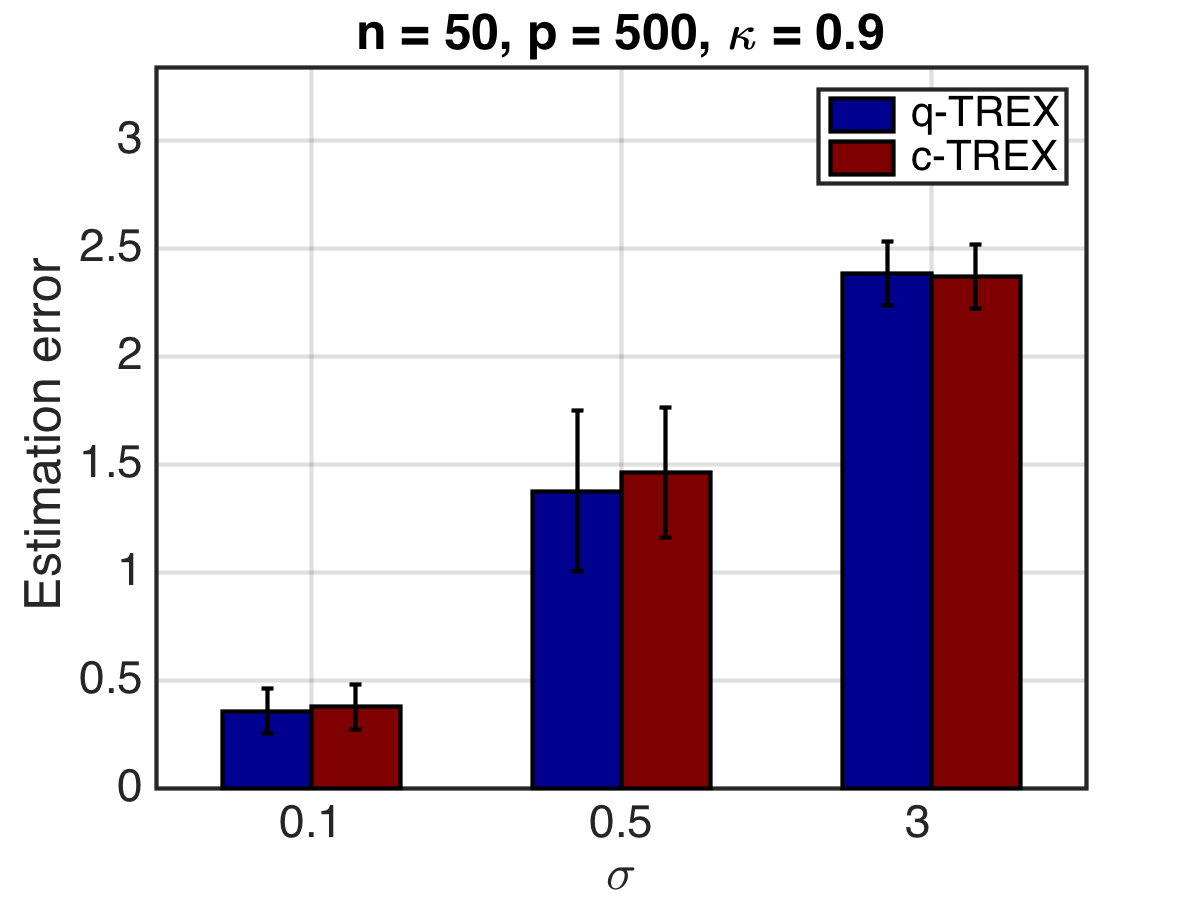}
  \caption{Average estimation error of q-TREX and c-TREX over $n_{rep}=21$ replications, $\kappa=0$ and $\kappa=0.9$.}
  \label{fig:estimation}
\end{figure}

\subsection{Topology of the Non-convex Objective}
\label{sec:topology}
While only the minimum of the $2p$ function values $P^*(s\consttrex x_j)$ is returned in the c-TREX algorithm, in this section we study the distribution of these function values and investigate whether this can give us deeper insight into the underlying problem regime. In Figure~\ref{fig:topology} we display histograms of the $2p$ optimal values computed in the c-TREX algorithm: $P^*(s\consttrex x_j)$ for $s\in\{-1,1\}$ and $j\in\{1,\ldots,p\}$.
%A surprising result is that the distribution of these function values can inform us on the underlying problem regime. 
%Specifically, the shape of the histogram of the $2p$ values is distinct in different regimes. In Figure~\ref{fig:topology} we present the histograms for the following
We repeat this in different problem regimes and find that the shape of the histogram of the $2p$ values differs according to problem regime. In particular, we consider the following
representative combinations of $(\kappa,\sigma)\in \{(0,0.1);(0.6,3);(0.9,0.5)\}$. 
%the other regimes 
We observe three histogram shapes arising:

% Figures with representative histograms as above + q-trex values
\begin{figure}[!t] % was H
    \centering
  \includegraphics[width=0.32\linewidth]{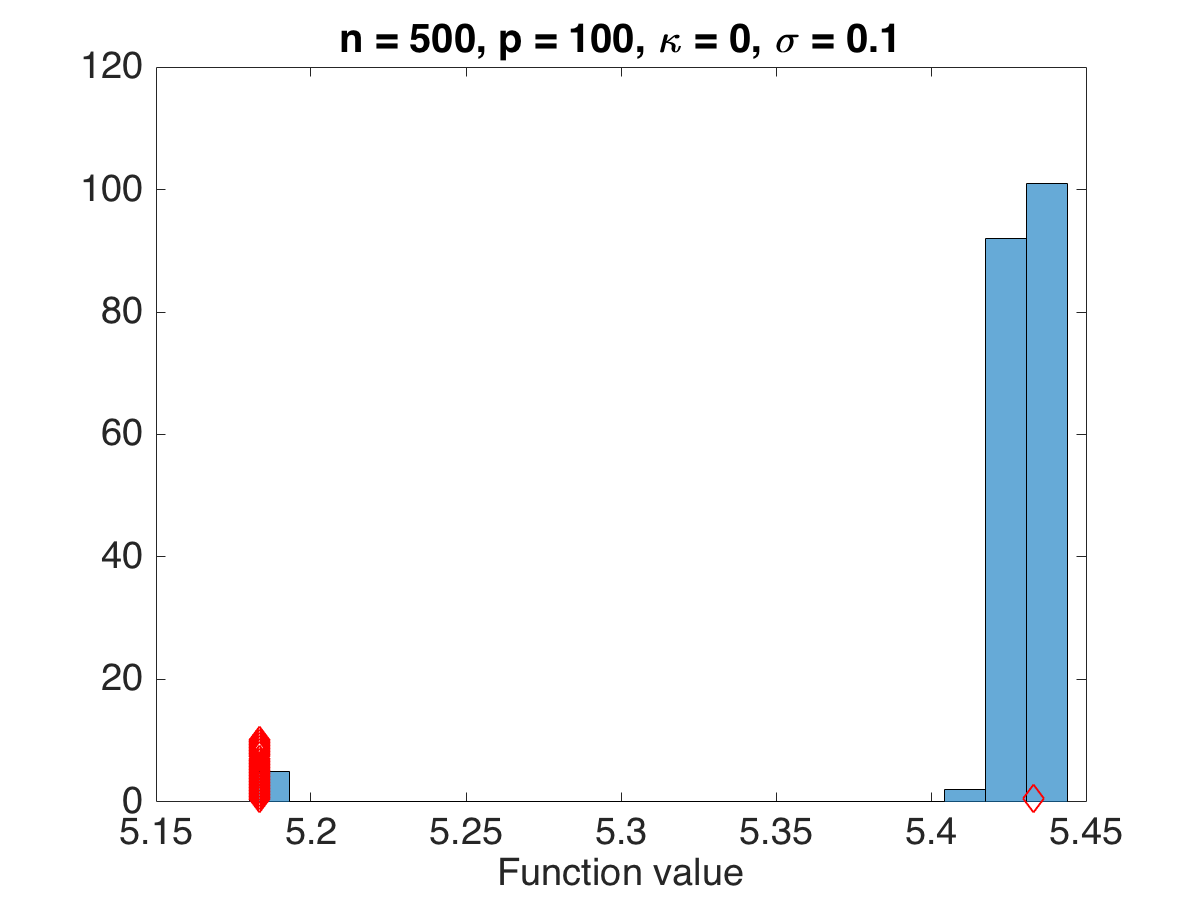}
    \includegraphics[width=0.32\linewidth]{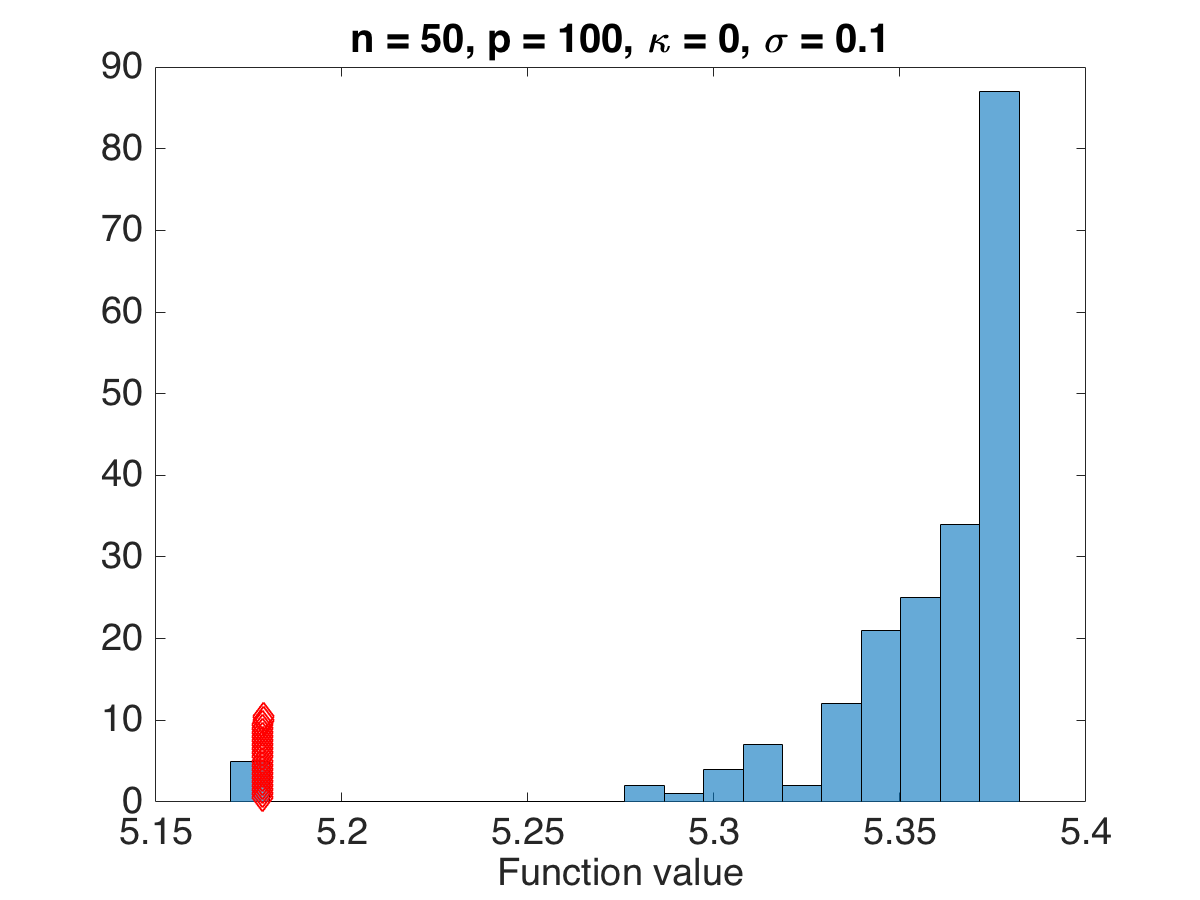}
  \includegraphics[width=0.32\linewidth]{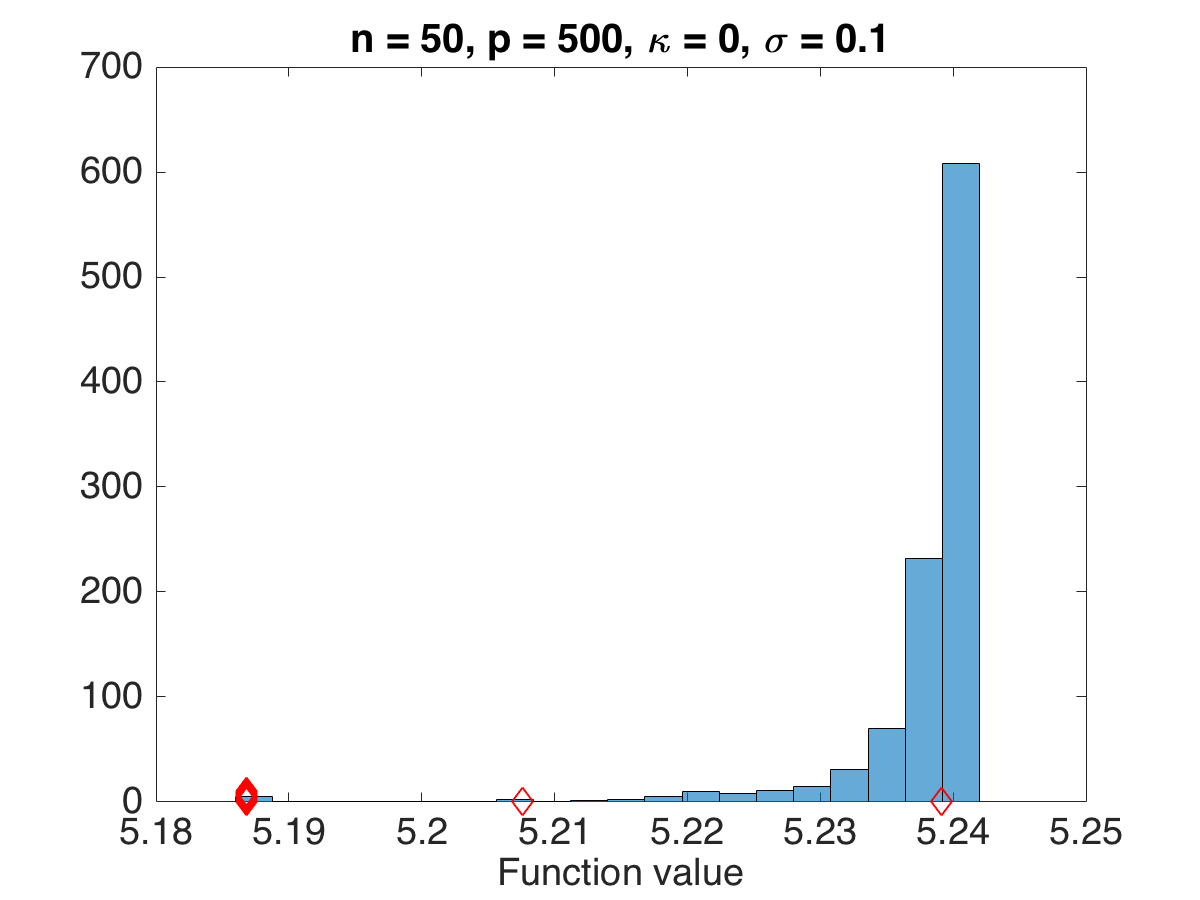}\\
  \includegraphics[width=0.32\linewidth]{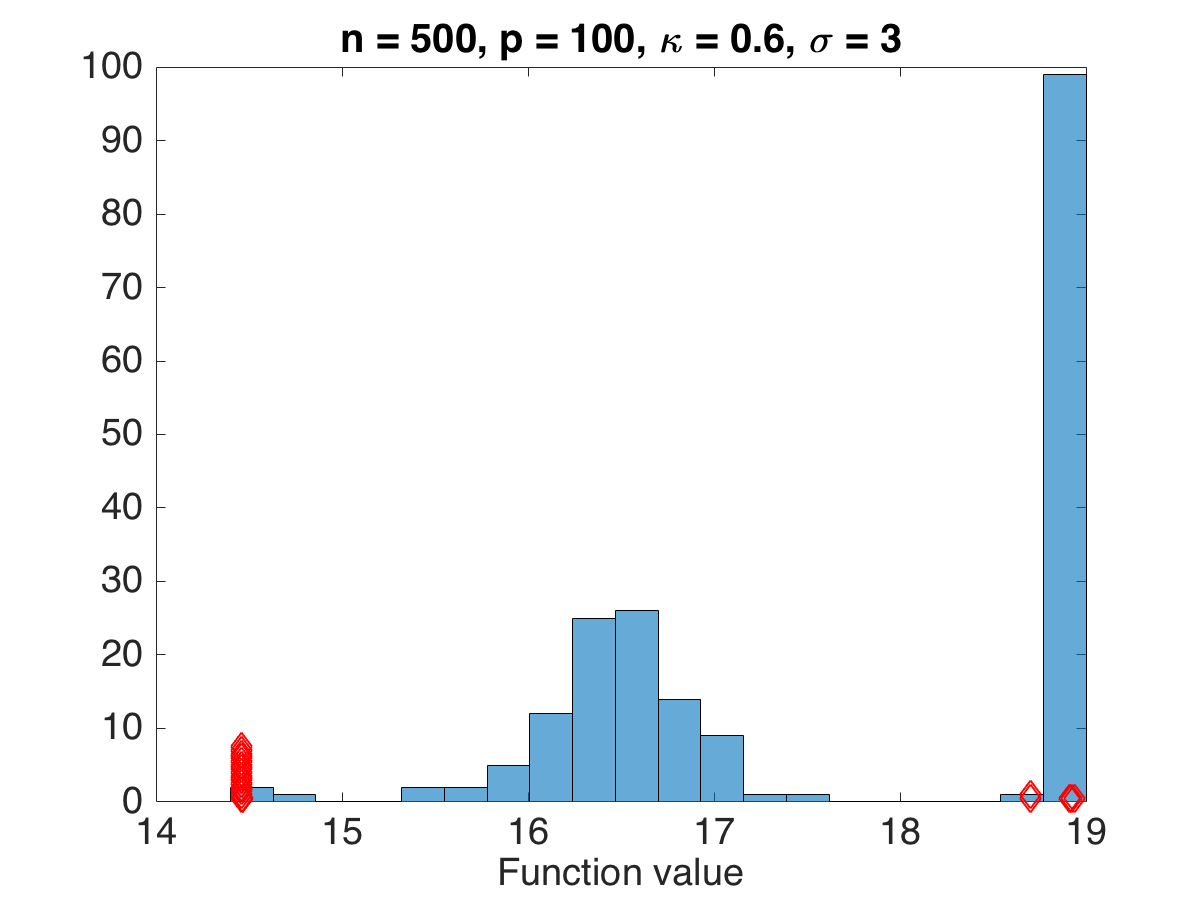}
    \includegraphics[width=0.32\linewidth]{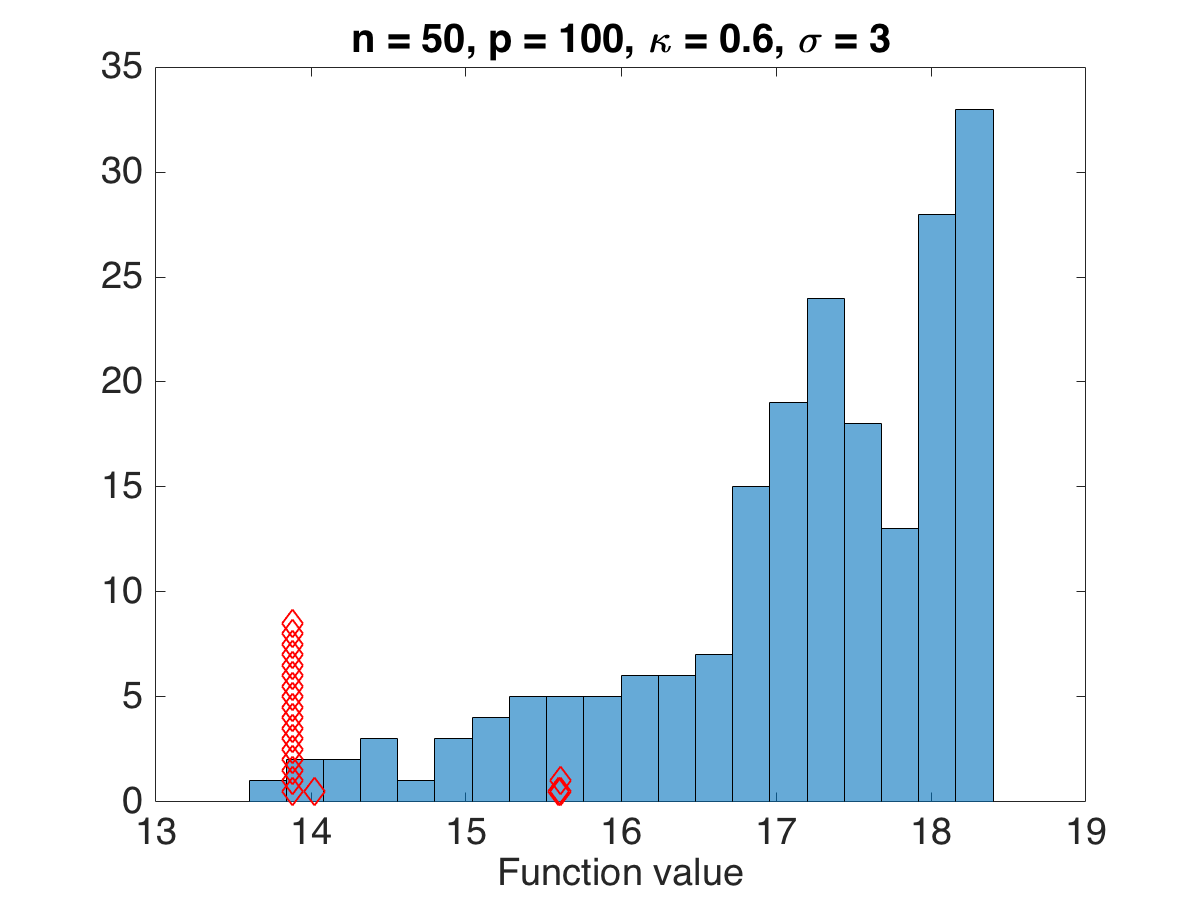}
  \includegraphics[width=0.32\linewidth]{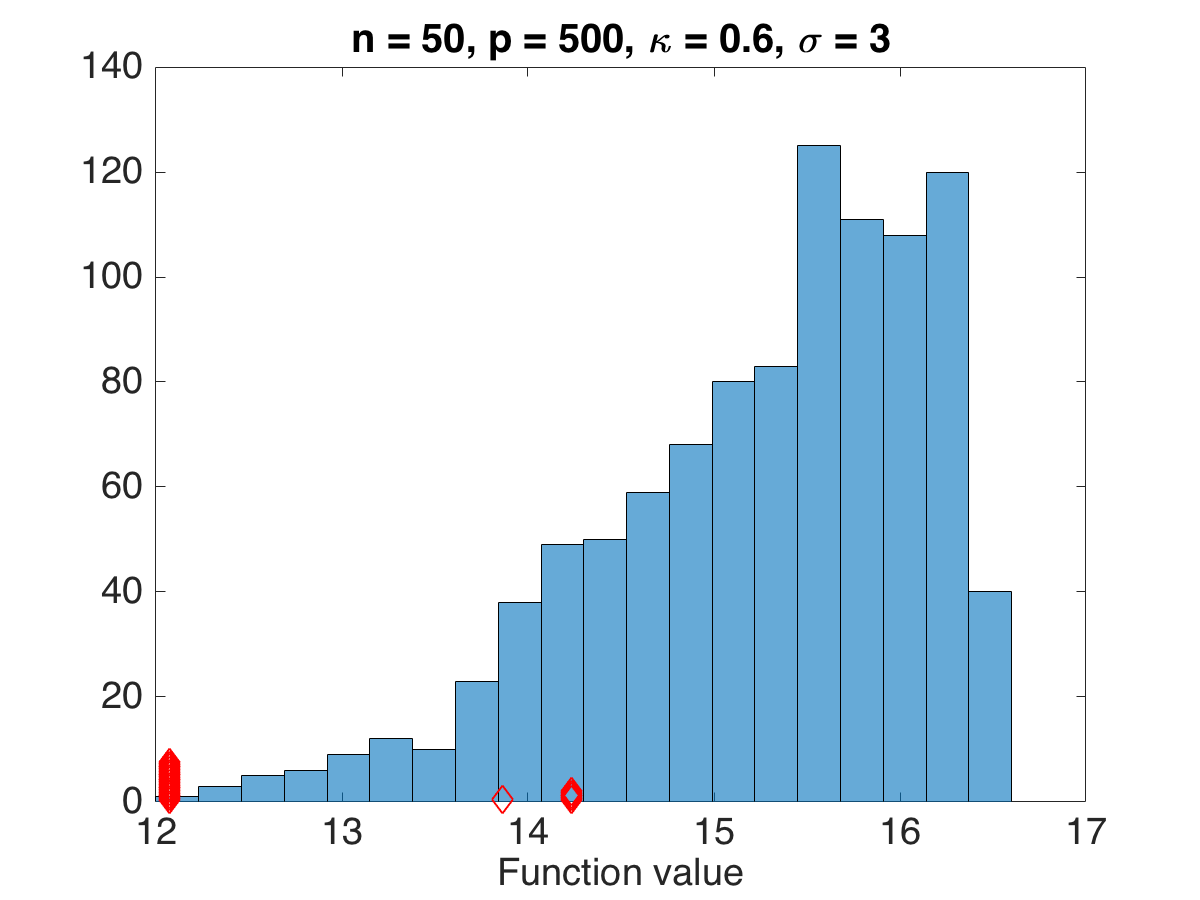}\\
  \includegraphics[width=0.32\linewidth]{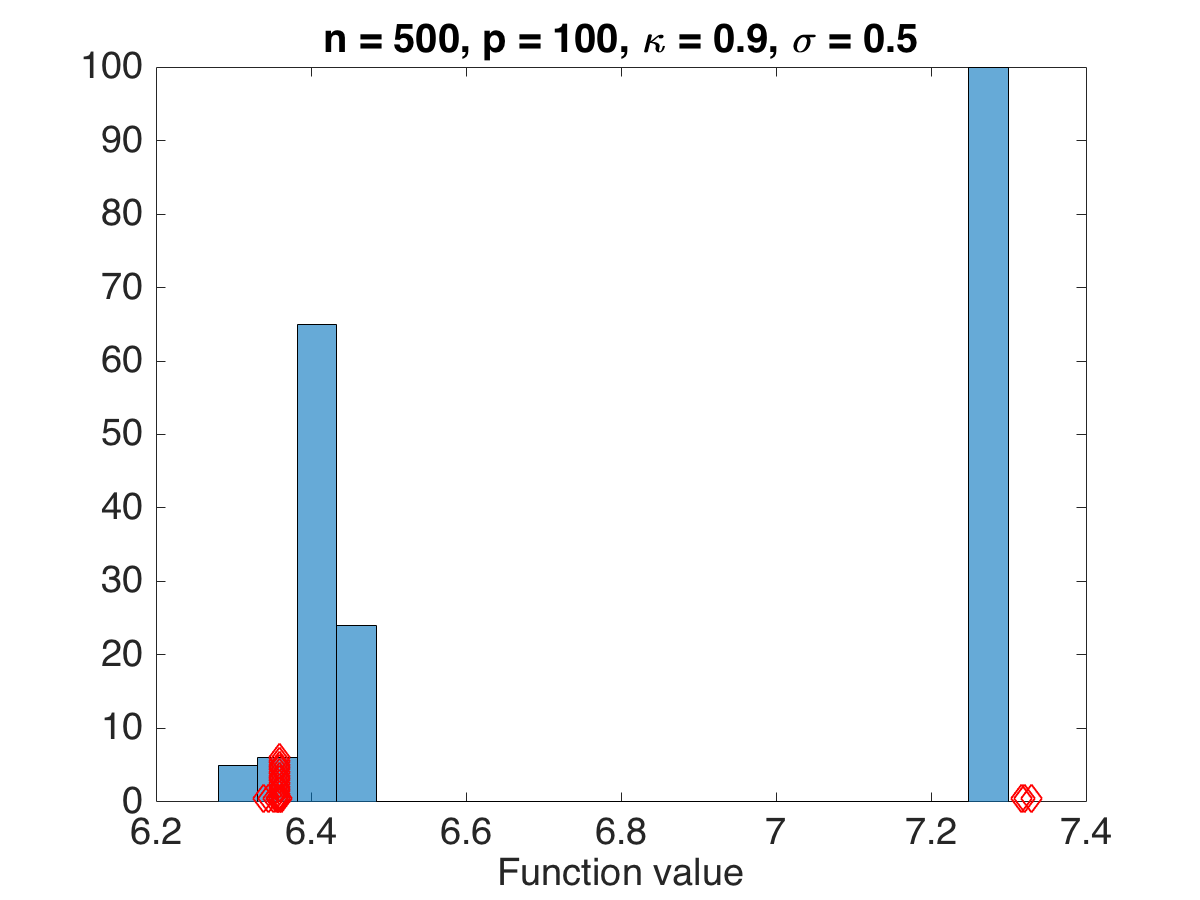}
    \includegraphics[width=0.32\linewidth]{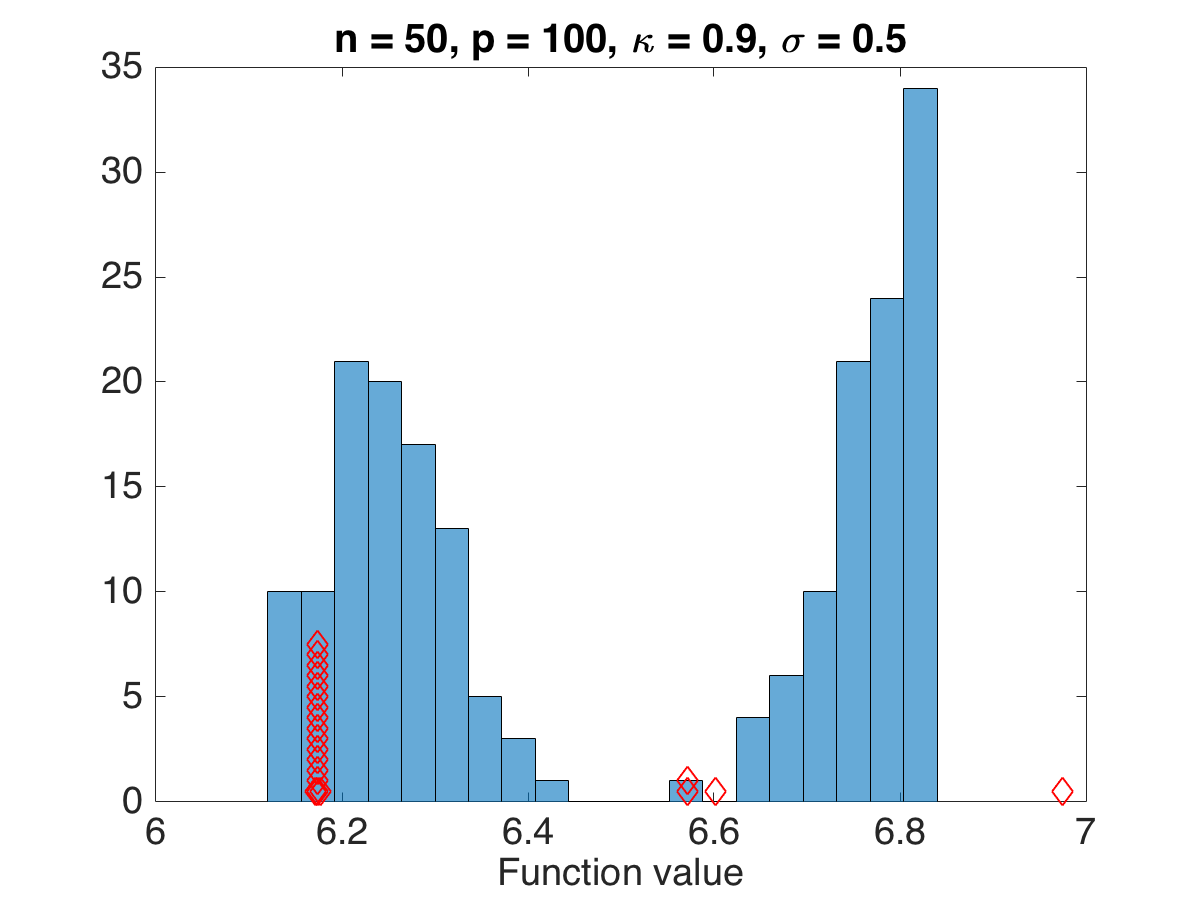}
  \includegraphics[width=0.32\linewidth]{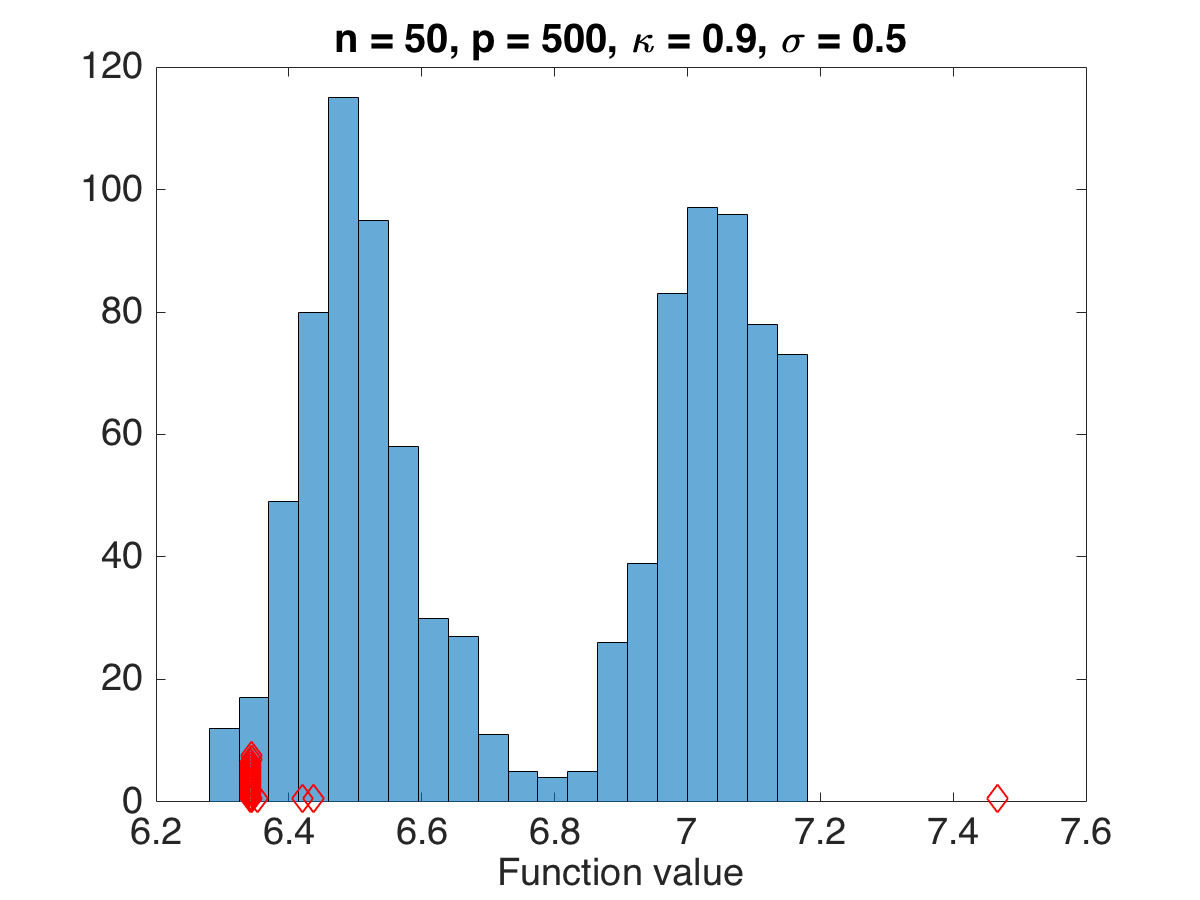}
  \caption{Histogram of $2p$ function values of c-TREX for one model instance. Diamond markers correspond to q-TREX function values from 21 random restarts, $(\kappa,\sigma)\in \{(0,0.1);(0.6,3);(0.9,0.5)\}$. The histogram for $p=500$, $\kappa=0.9$, $\sigma=0.5$ has an additional q-TREX marker at 14 (data point not shown).}
  \label{fig:topology}
\end{figure}

\begin{itemize}
    \item In the low $\kappa$, low $\sigma$ setting (top row of Figure~\ref{fig:topology}), the histogram is left-skewed and the global minimum values are clearly separated from the rest.
     
    \item In the moderate $\kappa$, high $\sigma$ setting (middle row), the histogram has a long left tail without clear separation between the values (unless $n$ is large).
    
    \item In the high $\kappa$ setting (bottom row), the histogram is bimodal regardless of the values of $n$, $p$ and $\sigma$.
\end{itemize}

Since the estimation error of both q-TREX and c-TREX strongly depends on the values of $(\kappa,\sigma)$, which are typically unknown, the above observations suggest that we can use the $2p$ function values from c-TREX to distinguish ``good" and ``bad" regimes in practice. 

Figure~\ref{fig:topology} also shows the function values from 21 random restarts of q-TREX superimposed (in red diamonds) on the histogram of the $2p$ function values of c-TREX. As expected from Figure~\ref{fig:prob-success}, in low $\kappa$ settings q-TREX attains the global minimum for the majority of starting points. In high $\kappa$ settings, q-TREX may fail to reach the objective value of c-TREX within $10^{-4}$ precision, with some starting points leading to objective values outside of the range of $2p$ values of c-TREX (see $p=500$, $\kappa=0.9$ and $\sigma=0.5$). 

Importantly, we also observe that in the low $\kappa$, low $\sigma$ setting, the indices $j$ of the $P^*_j=\min\{P^*(\consttrex x_j),~P^*(-\consttrex x_j)\}$ corresponding to the well-separated global or near-global optima coincide with the indices associated with non-zero $\beta_j$ in the true solution. Thus, inspection of these indices may give additional insights into which entries $\beta_j$ are potentially non-zero and motivates the use of $P^{*-1}_j$ as a measure of variable importance. In Section \ref{sec:knockoff}, we use this intuition to develop a procedure for false discovery rate control. 

\subsection{Topology of the TREX function on gene expression data}
\label{sec:topologyBsubtilis}
We next consider a real-world high-dimensional problem from genomics that has been introduced as a high-dimensional benchmark for linear regression in \cite{Buhlmann2014} and used in \citet{LedererMueller:14} to showcase q-TREX's ability to do meaningful variable selection. The c-TREX allows us to re-examine these previous results and analyze the topology of the TREX objective function in practice. The design matrix $X$ consists of $p = 4088$ gene expression profiles for $n = 71$ different strains of Bacillus subtilis (B. subtilis). The response $Y \in \real^{71}$ gives each strain's corresponding standardized riboflavin (Vitamin B) log-production rates. 
%The statistical task is to identify a small set of genes that is highly predictive of the riboflavin production rate in the statistical linear model. 
To get a rich picture of the topology, we solved the TREX problem ($\phi=0.5$) using the q-TREX ($q=40$) with $4088$ random restarts. We compared the resulting solutions to the $2p=8176$ c-TREX subproblems in terms of sparsity pattern and TREX function values. All numerical solutions have been thresholded at the level $\epsilon_t = 10^{-10}$ to discard ``numerical" zeros. Key results are summarized in Figure~\ref{fig:riboTopo}.        
\begin{figure}[!t] % was H
    \centering
  \includegraphics[width=0.94\linewidth]{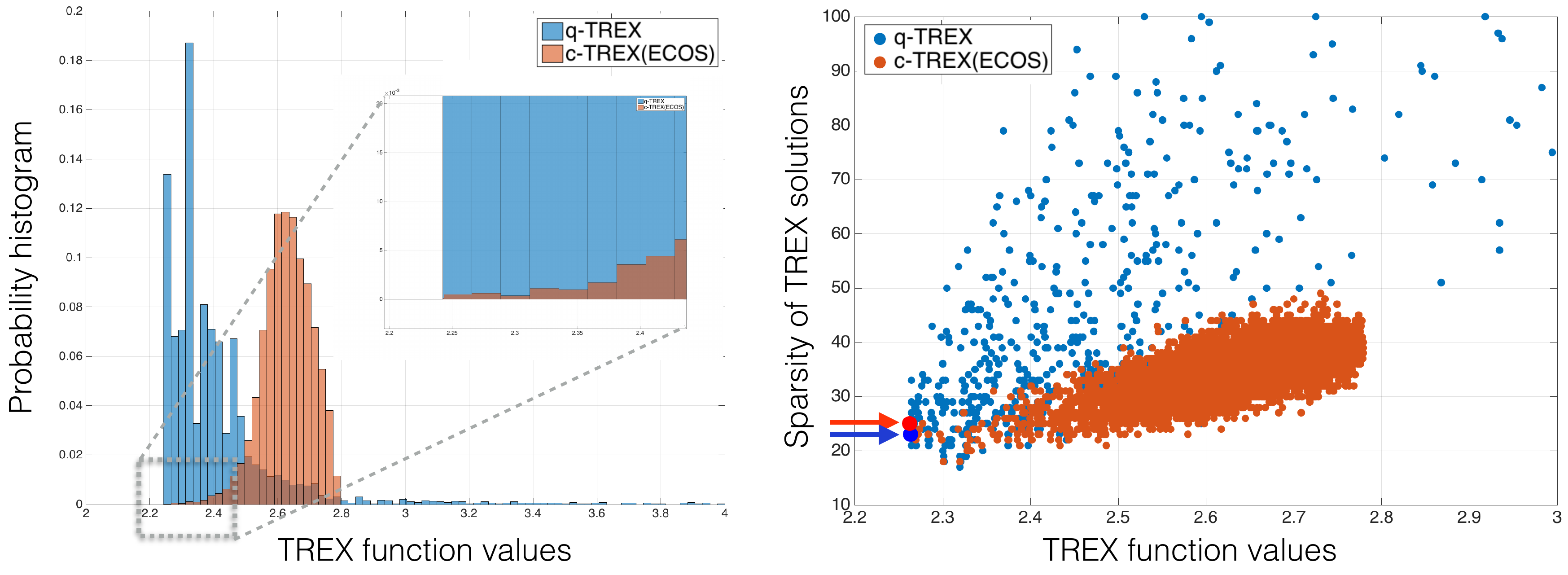}
  \caption{Topological properties of the TREX on the Riboflavin data. Left panel: Distribution of (locally) optimal function values for q-TREX (blue) and all $2p$ subproblems of c-TREX; Right panel: Solution sparsity vs. TREX function value. The optimal c-TREX solution has 25 non-zero entries (indicated by the red arrow), the optimal q-TREX solution has 23 non-zero entries (blue arrow).}
  \label{fig:riboTopo}
\end{figure}
The left panel shows histograms of the solutions from q-TREX and c-TREX. The shape of the histogram of c-TREX function values shows the previously observed long left tail, typical for moderate correlation in the design matrix and high variance. We observe that the histogram of q-TREX solutions is skewed right with the majority of solutions being close to the global c-TREX solution. For instance, $13\%$ of all q-TREX runs are within $10^{-3}$ of the global minimum $P^*$, suggesting that on the order of 10 %$O(10)$ 
q-TREX restarts may suffice to get a TREX solution within this tolerance. The right panel of Figure~\ref{fig:riboTopo} shows a strong relationship between function value and sparsity of the solution, especially for the c-TREX subproblems. The c-TREX global solution $\hat \beta_\text{c-TREX}$ has $s=25$ non-zero entries (red arrow Figure~\ref{fig:riboTopo}, right panel) compared to $s=23$ non-zeros in the best q-TREX solution $\hat \beta_\text{q-TREX}$. Inspection of the two solutions reveals a core of 20 common variables (genes). Both solutions achieve similar prediction error ($6.77$ with $\hat \beta_\text{c-TREX}$ and $6.82$ with $\hat \beta_\text{q-TREX}$ after least-squares refitting on the respective support). When we threshold the entries of $\hat \beta_\text{c-TREX}$ and $\hat \beta_\text{q-TREX}$ at the level $\epsilon_t = 10^{-4}$, both solutions have identical support with $s=17$ non-zero entries. All estimated coefficients, their gene names, as well as further analysis are found in the Appendix \ref{sec:Bsubtilis}. 

\section{Knockoff Filtering with the TREX}
\label{sec:knockoff}

Our ability to achieve exact global minimization of the TREX objective and to inspect the solutions of all $2p$ TREX subproblems provides us with a wealth of knowledge about the problem structure that can be potentially exploited in the design of novel statistical inference schemes. \citet{LedererMueller:14} provided empirical evidence that q-TREX is a competitive tuning-free variable selection alternative to the lasso when $p>n$. We investigate here whether one can achieve tuning-free variable selection with the TREX when there are at least as many observations $n$ available as variables $p$. In particular, we are interested in designing a variable selection scheme that controls the {\em false discovery rate} (FDR), i.e., the expected proportion of false variables among the selected variables.

For this purpose, we propose combining the TREX with the knockoff filtering framework, which is a recently developed approach for performing variable selection with FDR control in the context of the statistical linear model \citep{Barber2015}. 
\subsection{Background on the Knockoff Filter}
The principal idea of the knockoff filter is to create fake ``knockoff'' versions of the features and to have these fabricated features compete with the real features that they mimic.
More specifically, the knockoff filter procedure involves three steps: (i) efficiently generating an artificial data matrix $\tilde X\in\real^{n\times p}$ that closely matches the overall correlation structure of the actual data matrix $X\in\real^{n\times p}$; (ii) solving the linear model with the augmented design matrix $[X\, \tilde X]\in\real^{n\times 2p}$; and (iii) calculating feature-specific statistics (with certain properties) $W_j=W_j([X\, \tilde X], Y)$, where large values of $W_j$ indicate that the $j$th original feature is competing well against its knockoff, providing evidence against the null that $\beta_j=0$.

The artificial data matrix $\tilde X\in \R^{n\times p}$ is constructed so that
$$
\tilde X^{\top}\tilde X=X^{\top}X \quad \mbox{and} \quad X^{\top}\tilde X = X^{\top}X - \diag\{\bf s\}
$$
for some $p$-dimensional nonnegative vector ${\bf s}$. Increasing the elements of $\bf s$ allows the knockoff variables in $\tilde X$ to be more distinct from their counterparts in the original matrix $X$, leading to better statistical power. We choose $\bf s$ with $s_j = s$ for all $j$, which corresponds to the equi-correlated knockoffs discussed in \citet{Barber2015}.

The construction of the artificial data $\tilde X$ suggests that solving a linear model for the augmented data matrix $[X\, \tilde X]\in\real^{n\times 2p}$ should lead to similar small values of $\hat \beta_j$ and $\hat \beta_{j+p}$ when variable $j$ is not in the model, and different values of $\hat \beta_j$ and $\hat \beta_{j+p}$ when variable $j$ is in the model (with $\hat \beta_j$ being the larger in magnitude). This intuition is used for the construction of statistics $W_j$ with large positive values giving evidence against $\beta_j = 0$. For example, one can use $W_j = |\hat \beta^{LS}_j| - |\hat \beta^{LS}_{j+p}|$, where $\hat \beta^{LS}$ is the least-squares estimator. However, many other choices are possible in combination with different variable selection procedures. For the lasso procedure, a natural choice (the default \texttt{lassoSignedMax} setting in the published software package \citealt{knockoff2015}) is defined as follows: Let $Z_j = \sup \{\lambda: \hat \beta_j(\lambda)\neq 0\}$ for $j=1,\ldots,2p$ with $\hat \beta_j(\lambda)$ being the solution of the lasso for a given $\lambda$ value on the augmented problem regressing $Y$ on $[X\, \tilde X]$, then set $W_j = \max({Z_j,Z_{j+p}}) \sign(Z_j - Z_{j+p})$.

%to calculate a symmetric statistic $W([X\, \tilde X], Y) \in\real^{p}$ for each pair of original and knockoff variable $j$ where large positive values of $W_j$ give evidence that $\beta_j \neq 0$.
%To control the FDR at the nominal level $q_\text{FDR}$, \citet{Barber2015} introduce a data-dependent threshold for $W$% (see Equation 1.8 in \citealt{Barber2015}) 
%that leads to a ``knockoff" estimate of the underlying {\em false discovery proportion}.
%To generate valid knockoffs, \citet{Barber2015} introduced two constructions that lead to provable FDR control with the lasso: equi-correlated knockoffs (used in this study) and SDP knockoffs. We refer to \citet{Barber2015} for details on these constructions. They also give several instances of sufficient statistics $W$ both for step-wise forward selection and the lasso. A common choice (the default \texttt{lassoSignedMax} setting in the published software package \citealt{knockoff2015}) is defined as follows: Let $Z_j = \sup \{\lambda: \hat \beta_j(\lambda)\neq 0\}$ for $j=1,\ldots,2p$ with $\hat \beta_j(\lambda)$ being the solution of the lasso for a given $\lambda$ value on the augmented problem regressing $Y$ on $[X\, \tilde X]$. A valid statistic is then $W_j = \max({Z_j,Z_{j+p}}) \sign(Z_j - Z_{j+p})$. 

\subsection{Proposed TREX-based Knockoff Statistics}

Using this intuition, we next introduce two novel knockoff statistics that can be used in combination with the TREX. Our first proposal is a statistic similar to \texttt{lassoSignedMax}, in which we introduce a ``path" version of the TREX where the scalar $\consttrex$ in~\eqref{eq:trex} is varied from high to low values. We measure the quantity $Z^{\consttrex}_j = \sup \{\consttrex: \hat \beta_j(\consttrex)\neq 0\}$ for $j=1,\ldots,2p$ where $\hat \beta_j(\consttrex)$ is the solution of the TREX for a given $\consttrex$ value on the augmented problem. The associated statistic is $W^{\consttrex}_j = \max({Z^{\consttrex}_j,Z^{\consttrex}_{j+p}}) \sign(Z^{\consttrex}_j - Z^{\consttrex}_{j+p})$. A second more natural statistic is derived from the collection of $2p$ function values of c-TREX with standard value $\consttrex=0.5$. Recall that we have associated for each variable index $j$ the solution $P^*_j=\min\{P^*(\consttrex x_j),~P^*(-\consttrex x_j)\}$. We observe in Section~\ref{sec:topology} 
%The topological analysis of the TREX objective function has revealed
that the indices $j$ of (near-)optimal $P^*_j$ correspond to the indices of non-zero $\beta_j$ in the true solution. We thus propose the function-value associated TREX measure $Z^{f}_j = P^{*-1}_j$ with associated knockoff statistic $W^{f}_j = \max({Z^{f}_j,Z^{f}_{j+p}}) \sign(Z^{f}_j - Z^{f}_{j+p})$. This statistic thus takes full advantage of the entire topology of the non-convex TREX function on the augmented problem.   

\subsection{Theoretical Justification}

\citet[Theorems~1 and 2]{Barber2015} show that applying a data-dependent threshold to statistics $W_j$ leads to provable FDR control assuming these statistics satisfy two properties: \textit{antisymmetry} and \textit{sufficiency} \citep[see Definitions~3 and~4 in][]{Barber2015}. The antisymmetry property states that swapping $X_j$ and $\tilde X_j$ for any $j\in\{1,...,p\}$ has the effect of changing the sign of the $j$th statistic. The TREX-based statistics $W^\consttrex=(W^\consttrex_1,...,W^\consttrex_p)$ and $W^f=(W^f_1,...,W^f_p)$ satisfy this property by construction as $\sign(W_j^\consttrex)=\sign(Z_j^\consttrex-Z_{j+p}^\consttrex)$ and $\sign(W_j^f)=\sign(Z_j^f-Z_{j+p}^f)$.

The sufficiency property states that the statistic only depends on the Gram matrix $[X\, \tilde X]^{\top}[X\, \tilde X]$ and on the feature response inner products $[X\, \tilde X]^{\top}Y$. Unfortunately, this property does not hold for the TREX-based statistics due to an additional dependence on $\|Y\|_2$. 
However, we show in what follows that the results in~\citet{Barber2015}, specifically Theorems~1 and~2, hold true also under a relaxed version of the sufficiency property, which we define below (and, importantly, this property is satisfied by the statistics $W^\phi$ and $W^f$ for the TREX as shown in Lemma \ref{lem:trexS} in the Appendix).

\begin{definition}\label{d:gsuf}
 A statistic $W$ is said to obey the {\em generalized sufficiency property} if $W$ depends on $(X,\tilde X, Y)$ only through the Gram matrix $[X\, \tilde X]^{\top}[X\, \tilde X]$, the feature response inner products $[X\, \tilde X]^{\top}Y$, and the norm of the response vector $\|Y\|_2$, that is,
 $$
 W=f([X\,\tilde X]^{\top}[X\,\tilde X],[X\,\tilde X]^{\top}Y, \|Y\|_2).
 $$
\end{definition}
\noindent %The name \textit{feature sufficiency} property highlights that the restriction only applies to the relationship between $W$ and $X$. 

We call this the ``generalized sufficiency property'' since \citet{Barber2015}'s sufficiency property is included as a special case.
%The generalized sufficiency property is automatically implied by the sufficiency property of \citet{Barber2015}; however, the reverse is not necessarily true. Hence, \textit{generalized sufficiency} can be viewed as a relaxation of \textit{sufficiency}. 
The following theorem is the extension of \citet{Barber2015}'s Theorem~1 to this more general setting.
\begin{theorem}\label{t:FDR}
Let $W$ satisfy the antisymmetry and generalized sufficiency properties defined above. For any FDR target $q\in[0,1]$, define a data-dependent threshold $T$ as
\begin{equation*}
T=\min\left\{t\in \mathcal{W}:\frac{\#\{j:W_j\le -t\}}{\max(\#\{j:W_j\ge t\},1)}\le q \right\},
\end{equation*}
and chosen model $\hat S=\{j:W_j\ge T\}$. Here, $\mathcal W=\{|W_j|:j=1,\dots,p\}\setminus\{0\}$ is the set of unique non-zero absolute values of the elements of $W$. Then,
\begin{equation*}
\E\left[\frac{\#\{j:\beta_j=0\mbox{ and }j\in \hat S\}}{\#\{j:j\in \hat S\}+q^{-1}}\right]\le q.
\end{equation*}
\end{theorem}
\noindent 
The proof of Theorem~\ref{t:FDR} is deferred to Appendix \ref{sec:proof4_1}. This theorem provides theoretical support for the FDR control of the two proposed TREX-based knockoff filters. 
While our goal of course is to establish FDR control of the TREX knockoff filters, it should be noted that this new result also establishes FDR control for the square-root lasso \citep{sqrtlasso}, which does not satisfy the sufficiency property but does satisfy the generalized sufficiency property.  %In the next section, we provide empirical studies showing the practical potential of the TREX-based knockoff filters.

%\begin{remark} 
%While this generalization of the original knockoff framework is necessary to establish the valid FDR control for the TREX, it also generalizes the knockoff framework to other estimators. The square-root lasso \citep{sqrtlasso}, in particular, does not satisfy the sufficiency property, but does satisfy the generalized sufficiency property.
%\end{remark}

\section{Experiments with TREX-based Knockoff Filtering}
\label{sec:empirical-knockoff}
\subsection{Empirical Validation of TREX Knockoff Procedures}

\begin{figure}[!t] % was H
    \centering
  \includegraphics[width=0.99\linewidth]{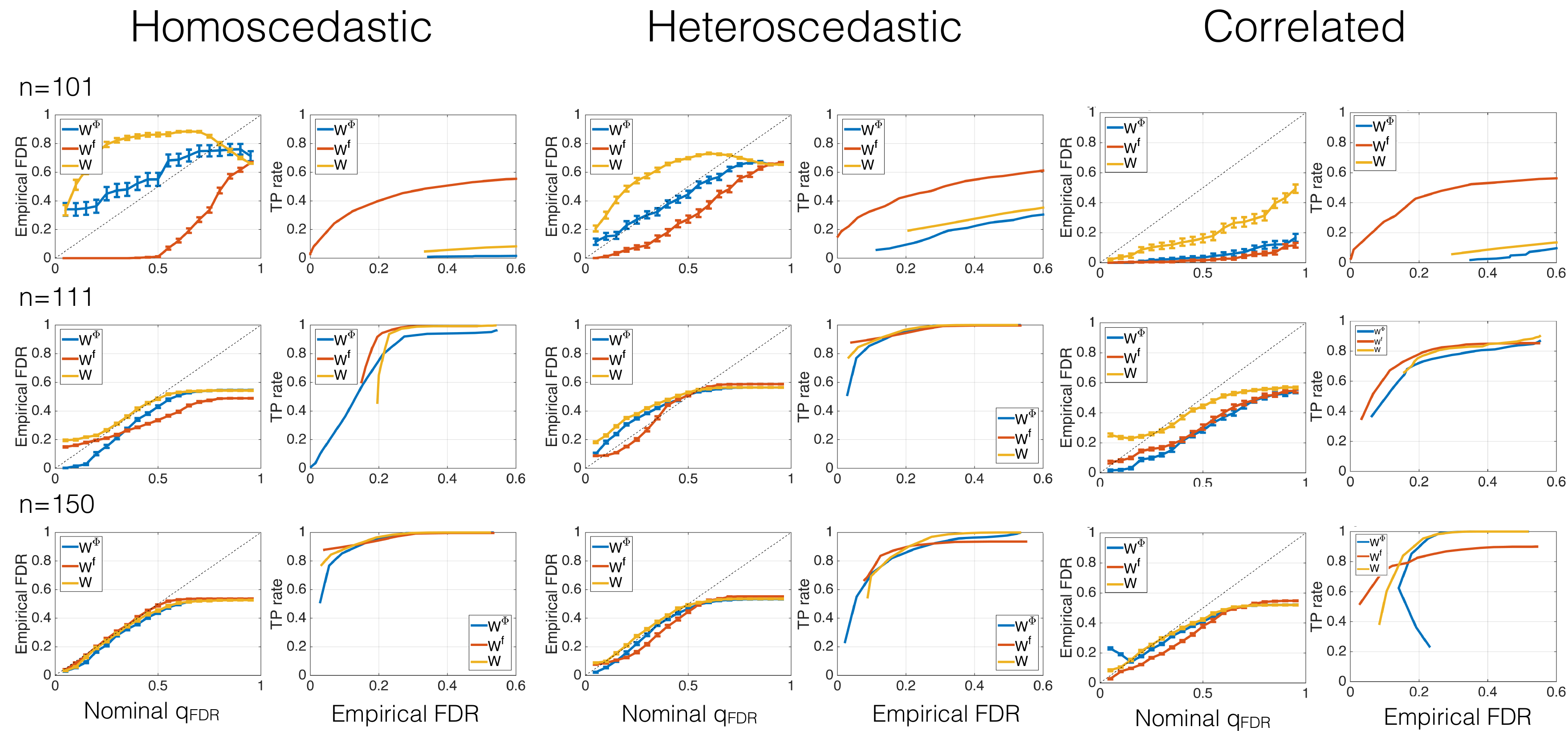}
  \caption{Mean and standard error of empirical FDR vs. nominal FDR level $q_\text{FDR}$, mean TP rates vs. empirical FDR under homoscedastic ($\sigma=1$) (left panel) and heteroscedastic noise ($\sigma_1=0.7$) (central panels), and correlated noise (right panels) for sample sizes $n\in\{101,111,150\}$ across all recorded FDR levels $q_\text{FDR}$ over 51 repetitions for TREX statistics $W^{f}$ (red) and $W^{\consttrex}$ (blue), and lasso-based statistic $W$ (\texttt{lassoSignedMax}) (yellow).}
  \label{fig:fdrTest}
\end{figure}

In %Section~\ref{sec:knockoff}, 
the previous section,
we proposed two TREX-based knockoff filters and proved that they provide FDR control.  To corroborate these results empirically, we follow the experimental setup in \citet{Barber2015} and simulate synthetic data according to the linear model $Y_i=X_i^{\top}\beta+\varepsilon_i$, $i=1,\dots,n$ with $p=100$, $s=30$ nonzero variables, regression vector $\beta=3.5(1_s^T, 0_{p-s}^T)^T$, and feature vectors $X_j\sim N(0,\Sigma)$ with $\Sigma_{jj}=1$, $j=1,\ldots,p$ and $\Sigma_{jk}=\kappa$ for $j\neq k$ with $\kappa=0.3$. We consider three noise scenarios: homoscedastic noise with errors $\varepsilon_i\sim N(0,\sigma^2)$ with $\sigma=1$, heteroscedastic noise with errors uniformly drawn from either $\varepsilon_i\sim N(0,\sigma_1^2)$ or $\varepsilon_i\sim N(0,\sigma_2^2)$ with $\sigma_1=0.7$ and $\sigma^2_2=2-\sigma_1^2$, and correlated noise with $\varepsilon_i\sim N(0,\Sigma^{\epsilon})$ with $\Sigma^{\epsilon}$ having the same structure as $\Sigma$. Each column $X_j$ is centered and normalized. We vary the sample size $n\in\{101,111,150\}$ and record the number of true positives (TPs) and the empirical FDR at nominal levels $q_\text{FDR}=\{0.05,0.1,\ldots,0.95\}$ for $r=51$ repetitions. We ran c-TREX with $\consttrex=0.5$ for the $W^{f}$ statistic and q-TREX over the $\consttrex$-path $\consttrex \in\{0.1,0.15,\dots,1.45,1.5\}$ for $W^{\consttrex}$. 
% Figures with representative histograms as above + q-trex values
We show empirical FDR and TP results when using the TREX knockoff filtering with statistics $W^{f}$, $W^{\consttrex}$, and lasso-based knockoff filtering with $W$ in Figure~\ref{fig:fdrTest}. 

Under homoscedastic noise, we observe that both novel statistics obey the nominal FDR at comparable power across most samples sizes and nominal levels. For the lowest possible sample size ($n=101$), only the $W^{f}$ statistic obeys the nominal FDR under all noise scenarios. Under heteroscedastic and correlated noise, the $W^f$ statistic has consistently higher TP rate than the other statistics for sample size $n=101$. For larger sample size, all statistics have similar power at comparable empirical FDR with $W^f$ being the most conservative.

\subsection{An Application to HIV-1 Data}

\begin{figure}[!t] % was H
    \centering
  \includegraphics[width=0.99\linewidth]{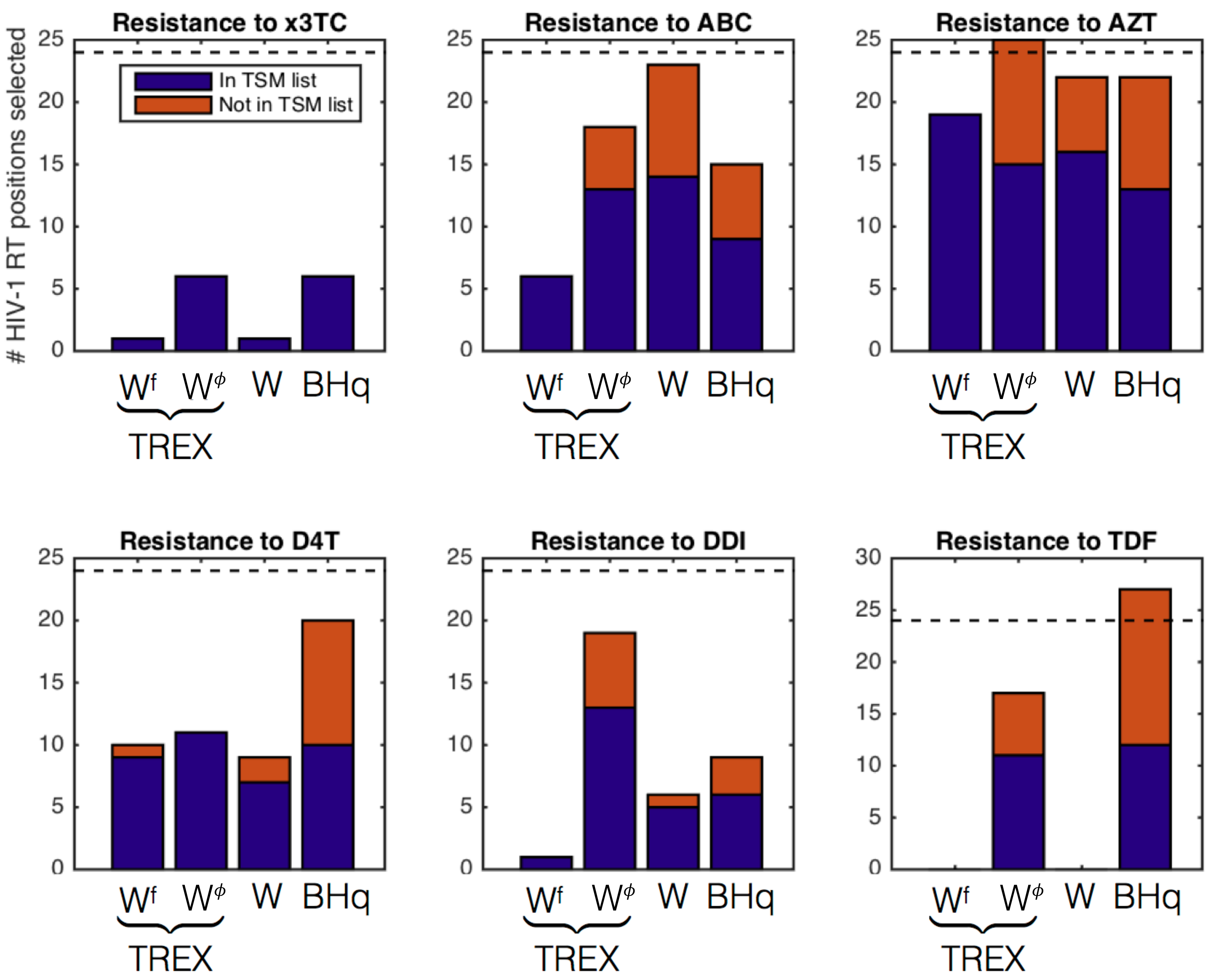}
  \caption{Number of positions on target NRTI for six drugs selected by TREX knockoff filtering with $W^{f}$ and $W^\consttrex$ statistic, lasso knockoff filtering with \texttt{lassoSignedMax} statistic $W$, and the Benjamini-Hochberg (BHq) method at target FDR $q_\text{FDR}=0.2$. The dashed horizontal line denotes the total number of positions in the TSM list. Data dimensions $n$ and $p$ for each experiment can be found in the lower panel of Figure~\ref{fig:HIV_PI}.}
  \label{fig:HIV_NRTI}
\end{figure}

We next apply TREX knockoff filtering to the task of inferring mutations in the Human Immunodeficiency Virus Type 1 (HIV-1) that are associated with drug resistance. The original data set \citep{Rhee2006} comprises drug resistance measurements and genotype information from samples of HIV-1 proteins. Separate data sets are available for resistance to six protease inhibitors (PIs), to six nucleoside reverse-transcriptase (RT) inhibitors (NRTIs), and to three non-nucleoside RT inhibitors (NNRTIs), respectively. The sample size of the different data sets ranges from $n=329$ to $843$ (see Figure~\ref{fig:HIV_PI} lower panel for details). Following \citep{Barber2015}, we analyze each drug separately using statistical linear models. The response $Y$ is given by log-fold changes of measured drug resistances. The design matrix $X$ with entries $X_{ij}\in\{0,1\}$ indicates absence or presence of (at least two) mutations at the $j$th genotyped position in the RT or protease, where distinct mutations at the same position are treated as additional separate features. In the absence of a ground truth for this real-world data set, we follow \citep{Barber2015} and compare the list of inferred mutations using knockoff filtering with lists of treatment-selected mutation (TSM) panels \citep{Rhee2006}. These TSM lists comprise all mutations that are present at significantly higher frequency in virus samples from previously treated patients compared with untreated control groups. Although these lists are target (RT, NRTI, NNRTI) but not drug specific, they still serve as a helpful proxy to the set of true positives across all tested drugs. We here apply TREX knockoff filtering with $W^{f}$ and $W^\consttrex$ statistic, lasso knockoff filtering with \texttt{lassoSignedMax} statistic $W$, and the Benjamini-Hochberg (BHq) procedure \citep{Benjamini1995} for target FDR $q_\text{FDR}=0.2$.             
\begin{figure}[!t] % was H

    \centering
  \includegraphics[width=0.94\linewidth]{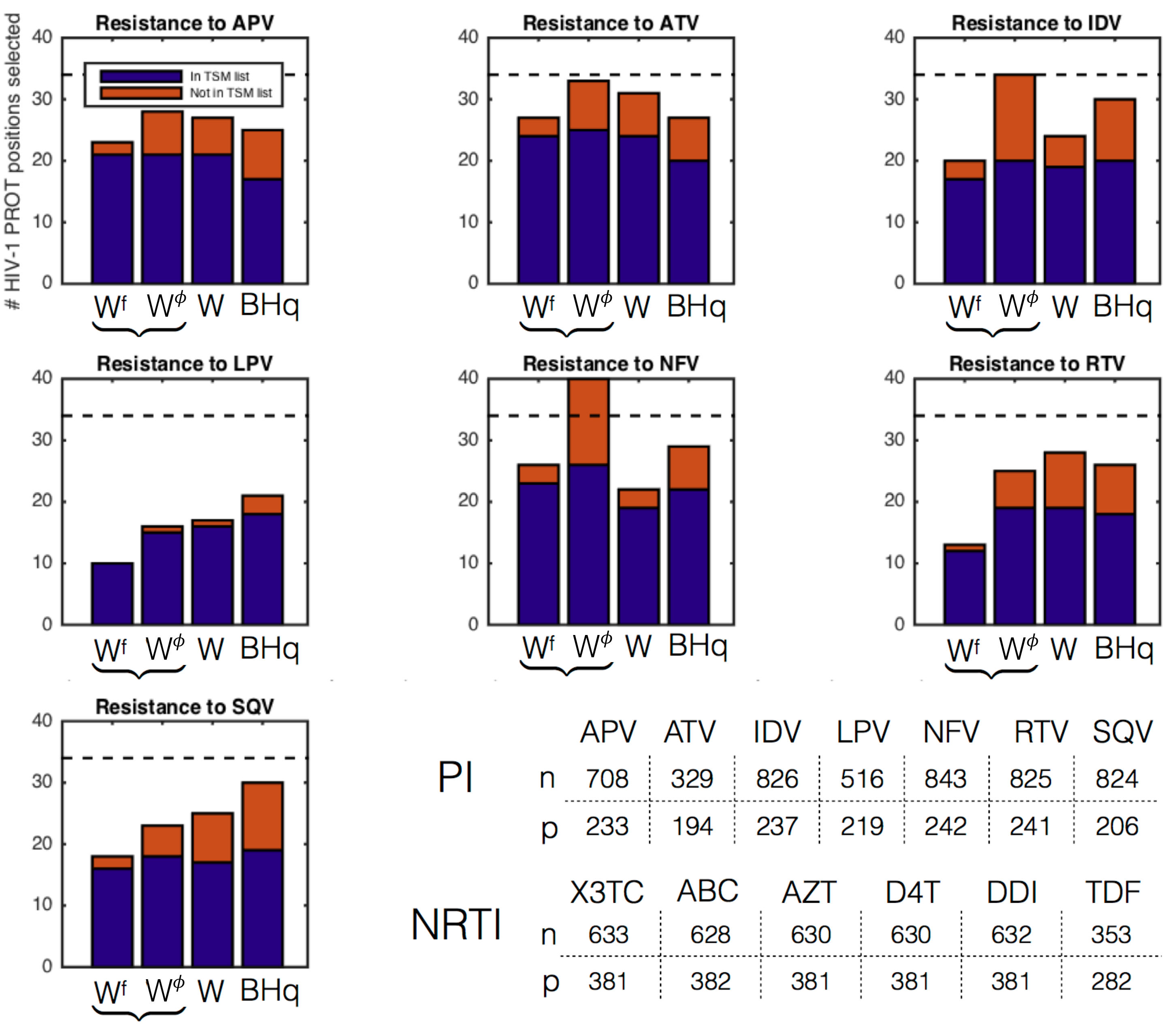}
  \caption{Number of positions on target PI for six drugs selected by TREX knockoff filtering with $W^{f}$ and $W^\consttrex$ statistic, lasso knockoff filtering with \texttt{lassoSignedMax} statistic $W$, and the Benjamini-Hochberg (BHq) method at target FDR $q_\text{FDR}=0.2$. The dashed horizontal line denotes the total number of positions in the TSM list. The bracket marks the two TREX-based statistics. The lower panel shows the dimensionality of each problem.}
  \label{fig:HIV_PI}
\end{figure}

Figure~\ref{fig:HIV_NRTI} summarizes the TSM recovery performance of all methods for mutations in NRTI across all six drugs. We observe several differences among the methods. First, the TREX with $W^{f}$ statistic selects only mutations from the TSM list across all drugs (except one additional for D4T). For x3TC and TDF it selects one or no mutation (identical to the performance of the lasso). The TREX shows remarkable performance for AZT (the first successful NRTI drug) where it recovers 19 out of 24 mutations without reporting any mutation outside the TSM list. 
Using TREX with selection along the $\consttrex$-path and $W^{\consttrex}$ results in a less conservative variable selection procedure. Moreover, it compares favorably to its lasso analog with increased power on TDF, x3TC, and DDI. The BHq and TREX with $W^\consttrex$ are the only methods that select mutations for the drug TDF.

We next analyze TSM recovery performance of all methods for mutations in PI across seven tested drugs (see Figure~\ref{fig:HIV_PI}). We observe a similar trend as in the NRTI example. TREX Knockoff filtering with $W^{f}$ recovers mostly variables from the TSM list at similar power compared to all other methods. The performance of the TREX with $W^\consttrex$ statistic is again comparable to knockoff filtering with the lasso. Finally, on the NNRTI test case (data not shown) all methods show similar variable selection behavior with about one half of the predictions present in the TSM lists with additional novel mutational positions.

\section{Discussion}
\label{sec:discussion}

In this paper, we introduce a new algorithm, called c-TREX, that is guaranteed to attain the global minimum of the non-convex TREX problem.  Having access to the true global minimum is extremely rare in non-convex optimization.  We use this new ability to investigate the performance of a previously proposed heuristic, the q-TREX, in a way that is typically impossible in other non-convex problems.  We observe that q-TREX's success in attaining the global minimum is affected by various parameters of the underlying model such as the error variance and the correlation between features.  We do, however, observe that in terms of statistical performance the c-TREX and q-TREX estimators are on par, suggesting that q-TREX's sub-optimality in terms of the TREX objective may not negatively affect its performance as an estimator.  

The c-TREX algorithm involves solving $2p$ separate convex problems that are based on the original TREX problem. The convex problems belong to the class of second-order cone programs (SOCPs). We have used two state-of-the-art SOCP solvers: ECOS (Embedded Conic Solver, \citealt{domahidi2013ecos}), an interior-point solver, and SCS (Splitting Conic Solver, \citealt{odonoghue2015conic}), a first-order method. Our empirical investigations show that these solvers are able to solve c-TREX within reasonable time but are not yet competitive with the q-TREX heuristic. An interesting line of future research is thus to design dedicated SOCP solvers that use the special structure of the TREX problem as well as proximal algorithms \citep{Combettes2016}.     

Our analysis of the TREX problem landscape shows that having access to all $2p$ TREX solutions is a rich source for insight about the underlying model. We observe that the ``topology" of these solutions appears to differ in an informative way depending on the problem regime. We observe that (i) the distribution of $2p$ function values associated with the solutions becomes increasingly multi-modal with statistical problem difficulty and (ii) the $2p$ function values permit a novel ranking scheme for variable importance.

Another major contribution of this work is that we show that the knockoff filter can be applied to the TREX, leading to two new procedures for controlling the FDR for variable selection.  One of our knockoff statistics that performs particularly well makes explicit use of the $2p$ solutions computed in the c-TREX algorithm.  Our empirical study corroborates that FDR is controlled at the nominal level and offers promising evidence on synthetic and real-world data that a strong ability to detect true positives is maintained.  

%\section*{Acknowledgments}

%The authors thank three referees and an AE for helpful comments that improved this paper and are particularly grateful to one referee who provided an important fix to %the proof of FDR control.  They also thank Alexander Domahidi (of Embotech) for help with the ECOS software.

\bigskip
\begin{center}
{\large\bf ONLINE MATERIAL}
\end{center}

\begin{description}

%\item[Title:] Brief description. (file type)

\item[MATLAB-package for TREX routine:] github.com/muellsen/TREX
%\item[S1:] Proof of Theorem 4.1
%\item[S2:] Further analysis of TREX on the B. subtilis data.

\end{description}

\bibliographystyle{agsm}
\bibliography{references}

\begin{appendix}

\bigskip
\begin{center}
{\large\bf APPENDIX}
\end{center}

\section{Generalized Sufficiency Property of \texorpdfstring{$W^\consttrex$}~~and \texorpdfstring{$W^f$}~}

\begin{lemma}\label{lem:trexS}
The statistics $W^\consttrex$ and $W^f$ can be written as
$$
W^\consttrex=g^\consttrex\left([X\, \tilde X]^{\top}[X\, \tilde X],[X\, \tilde X]^{\top}Y,\|Y\|_2\right)\quad\mbox{and}\quad W^f=g^f\left([X\, \tilde X]^{\top}[X\, \tilde X],[X\, \tilde X]^{\top}Y,\|Y\|_2\right)
$$
for some $g^\consttrex:S^+_{2p}\times R^{2p}\times R\to R^p$ and $g^f:S^+_{2p}\times R^{2p}\times R\to R^p$, where $S^+_{2p}$ is the cone of $2p\times 2p$ positive semidefinite matrices.
\end{lemma}
\begin{proof}
The TREX criterion~\eqref{eq:trex} can be rewritten as 
\begin{align*}
    \min_{\beta\in\real^p}\left\{\frac{\|Y\|_2^2-2\beta^{\top}X^{\top}Y+\beta^{\top}X^{\top}X\beta}{\|X^\top Y-X^{\top}X\beta\|_\infty}+\consttrex\|\beta\|_1\right\},
\end{align*}
where the objective function depends on the data only through $X^{\top}X$, $X^{\top}Y$, and $\|Y\|_2$. By construction, this is also true for each convex subproblem~\eqref{eq:convex}.
\end{proof}
 %Before we introduce the main theorem, we prove the following auxillary lemma:

\spacingset{1.45}% DON'T change the spacing!'T change the spacing!

\section{Proof of Theorem~4.1}
\label{sec:proof4_1}
\begin{lemma}\label{l:Rswap}
There exists an orthogonal matrix $R\in \R^{n\times n}$ such that
$$
R[X\, \tilde X] = [X\, \tilde X]_{\swap(S)}.
$$
\end{lemma}
\begin{proof}[Proof of Lemma~\ref{l:Rswap}]
Define the square matrix $M:=[X\, \tilde X]_{\swap(S)}[X\, \tilde X]^{\top}$, and consider its full singular value decomposition $M=U_mD_mV_m^{\top}$, with $D_m$ being a square diagonal matrix of singular values. We show that an orthogonal matrix $R=U_mV_m^{\top}$ satisfies the conditions of the lemma: $R[X\, \tilde X] - [X\, \tilde X]_{\swap(S)} = 0$.
%First we show that $R$ is an orthogonal matrix. Since $U_m$ and $V_m$ are left and right singular vectors from the full singular value decomposition of $M$, it follows that both $U_m$ and $V_m$ are square orthogonal matrices. Therefore, $R=U_mV_m^{\top}$ is also orthogonal.

Consider the Frobenius norm 
\begin{align*}
\|R&[X\,\tilde X]-[X\,\tilde X]_{\swap(S)}\|^2_F\\
&=\Tr\left([X\,\tilde X]^{\top}R^{\top}R[X\,\tilde X]+[X\,\tilde X]_{\swap(S)}^{\top}[X\,\tilde X]_{\swap(S)}-2R^{\top}[X\, \tilde X]_{\swap(S)}[X\, \tilde X]^{\top}\right)\\
&=\Tr(2[X\,\tilde X]^{\top}[X\,\tilde X]-2 V_mU_m^{\top}U_mD_mV_m^{\top})\\
&=\Tr(2[X\,\tilde X]^{\top}[X\,\tilde X]-2 D_m),
\end{align*}
where in the second equation we used the definition of $R$, and the invariance of the Gram matrix $[X\,\tilde X]^{\top}[X\,\tilde X]$ under the $\swap(S)$ operation. From the above display, the Frobenius norm is zero iff $\Tr([X\,\tilde X]^{\top}[X\,\tilde X])=\Tr(D_m)$. We conclude the proof by showing that the latter inequality holds.

Consider the full singular value decomposition of $[X\,\tilde X] = U_xD_xV_x^{\top}$. Using the definition of matrix $M$, and the invariance of Gram matrix $[X\,\tilde X]^{\top}[X\,\tilde X]$ under the $\swap(S)$ operation,
\begin{align*}
M^{\top}M&=[X\, \tilde X][X\, \tilde X]_{\swap(S)}^{\top}[X\, \tilde X]_{\swap(S)}[X\, \tilde X]^{\top} \\
&= [X\, \tilde X][X\, \tilde X]^{\top}[X\, \tilde X][X\, \tilde X]^{\top}\\
&= U_xD_xV_x^{\top}V_xD_xU_x^{\top}U_xD_xV_x^{\top}V_xD_xU_x^{\top}\\
&= U_xD_x^4U_x^{\top}.
\end{align*}
On the other hand, using the full singular value decomposition of $M$,
$$
M^{\top}M=V_mD_m^2V_m^{\top}.
$$
It follows that $V_mD_m^2V_m^{\top}=U_xD_x^4U_x$, and both equations can be viewed as eigendecomposition of $M^{\top}M$. Hence, $D_x^{4}=D_m^2$ up to the permutation, and subsequently $\Tr(D_m) = \Tr(D_x^2)$. Since $\Tr([X\,\tilde X]^{\top}[X\,\tilde X])=\Tr(D_x^2)=\Tr(D_m)$, this concludes the proof.
\end{proof}

\begin{proof}[Proof of Theorem~4.1]%\ref{t:FDR}]
In \citet{Barber2015}, the sufficiency property is only used in the proof of Lemma~1, where it is shown that for any subset $S$ of null features $N = \left\{j\in\{1,\dots,p\}: \beta_j = 0\right\}$,
$$
W_{\swap(S)}\stackrel{d}{=}W,
$$
where $\stackrel{d}{=}$ means equality in distribution and $\swap(S)$ swaps the columns $X_j$ and $\tilde X_j$ in $[X\, \tilde X]$ for each $j\in S$. Further we show that this also holds for generalized sufficiency property.

By Lemma~\ref{l:Rswap}, there exists an orthogonal matrix $R\in\R^{n\times n}$ such that
\begin{equation}\label{eq:Rswap}
R[X\, \tilde X] = [X\, \tilde X]_{\swap(S)}.
\end{equation}
Therefore, 
\begin{align*}
W_{\swap(S)}& =f([X\, \tilde X]_{\swap(S)}^{\top}[X\, \tilde X]_{\swap(S)}, [X\, \tilde X]_{\swap(S)}^{\top}Y, \|Y\|_2)\\
& =f([X\, \tilde X]^{\top}[X\, \tilde X], [X\, \tilde X]^{\top}R^{\top}Y, \|R^{\top}Y\|_2)\\
& = f([X\, \tilde X]^{\top}[X\, \tilde X], [X\, \tilde X]^{\top}Y', \|Y'\|_2)
\end{align*}
with $Y' :=R^{\top} Y$. Hence, $W_{\swap(S)}$ can be viewed as $W$ applied to the same $[X\, \tilde X]$ and modified response $Y'$. Since $[X\, \tilde X]$ is fixed, it follows from above that to show $W_{\swap(S)}\stackrel{d}{=}W$, it is sufficient to show $Y\stackrel{d}{=} Y'$.

Since $Y\sim \Ncal(X\beta,\sigma^2I)$, 
$$
Y'=R^{\top}Y\sim \Ncal(R^{\top}X\beta, \sigma^2 R^{\top}R)=\Ncal(R^{\top}X\beta, \sigma^2 I),
$$
where we used the orthogonality of matrix $R$. From above, to show $Y\stackrel{d}{=} Y'$, it remains to show $R^{\top}X\beta = X\beta$.

Denote $[X\, \tilde X]_{\swap(S)} = [X_{\swap(S)}\, \tilde X_{\swap(S)}]$. Since $S$ only contains null features, $X\beta = X_{\swap(S)}\beta$, hence
$$
R^{\top}X\beta = R^{\top}X_{\swap(S)}\beta.
$$
On the other hand,~\eqref{eq:Rswap} implies $RX = X_{\swap(S)}$. Performing left multiplication by $R^{\top}$ on both sides gives $X = R^{\top}X_{\swap(S)}$, and subsequently $X\beta = R^{\top}X_{\swap(S)}\beta$. Combining this with the above display gives
$
R^{\top}X\beta = X\beta
$
completing the proof.
\end{proof}

\section{Further analysis of TREX on the B. subtilis data}
\label{sec:Bsubtilis}
This section provides additional information about the solution quality of the c-TREX globally optimal and the q-TREX solution. Figure~\ref{fig:suppl1} shows key results. We first report the correlation structure of the selected features for the q-TREX and c-TREX solutions, along with the corresponding B. subtilis gene names. We observe two blocks of correlated genes: the set XHLB, XKDS, XLYA, and XTRA, and the pair YDDK, YDDM, (Figure~\ref{fig:suppl1}, top panels), suggesting that the set of genes might be further reducible, e.g., via empirical Bayes approaches \citep{Bar2015}. We also show least-squares refits on the support of the q-TREX and c-TREX solution to the measured  log-production rate (Figure~\ref{fig:suppl1}, lower left panel). 
\begin{figure}[H] % was H
    \centering
  \includegraphics[width=0.99\linewidth]{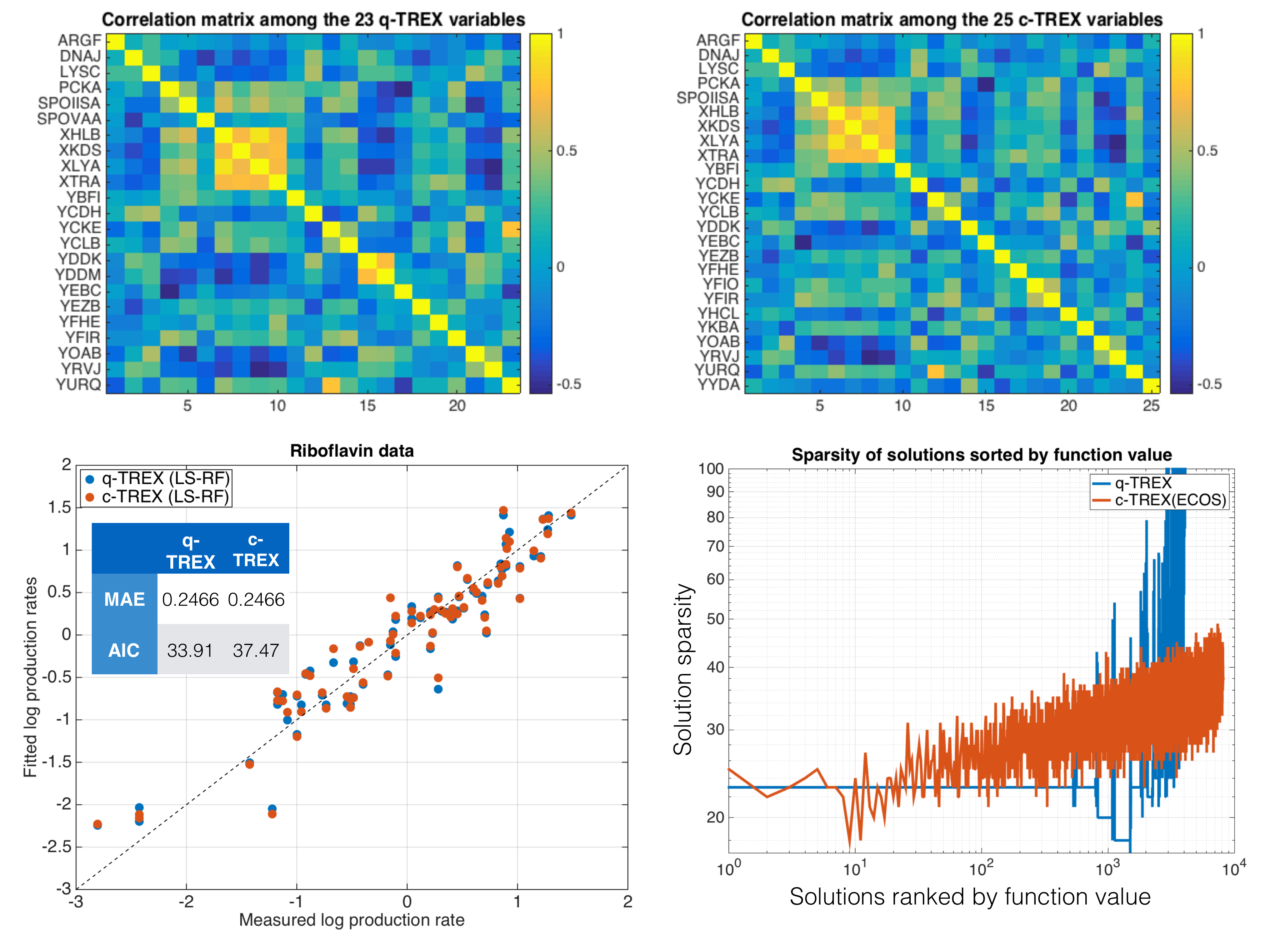}
  \caption{Top panels: Correlation matrix of the features (genes) selected by q-TREX (left panel) and c-TREX (right panel); Lower left panel: Fitted log-production rate vs. measured riboflavin log-production rate. The table insert shows the mean absolute error (MAE) and the Akaike Information Criterion (AIC) for the q-TREX and c-TREX solutions; Lower right panel: Sparsity of all q-TREX restart solutions and c-TREX subproblems vs. ranking of the solution in terms of TREX function value.}
  \label{fig:suppl1}
\end{figure}

While both solutions yield similar mean absolute error (MAE), the AIC of the q-TREX compares favorably to that of the c-TREX global minimum due to lower model complexity. In terms of AIC, both solutions improve upon solutions found by Lasso with stability selection \citep{Buhlmann2014} and an empirical Bayes approach \citep{Bar2015} and are comparable in terms of MAE. For completeness, we also report the sparsity of all found solutions vs. their ranking in terms of function values (Figure~\ref{fig:suppl1}, lower right panel). We observe that more than the top $20\%$ of the q-TREX solutions have sparsity $s=23$ whereas the sparsity of the top $20\%$ of the c-TREX solutions vary between 18 and 40.

\end{appendix}

\end{document}